\documentclass[conference]{IEEEtran}

\IEEEoverridecommandlockouts                              

\usepackage{multirow}

\usepackage{graphicx}
\usepackage{cite}
\usepackage{xcolor}
\usepackage{hyperref}
\usepackage[per-mode=symbol]{siunitx}
\usepackage{proof}
\usepackage{bussproofs}
\usepackage{mathpartir}
\usepackage{rotating}
\usepackage{booktabs}
\usepackage{subcaption}

\usepackage{algorithm2e}
\SetKwComment{Comment}{/* }{ */}
\RestyleAlgo{ruled}

\usepackage[all]{xy} \SelectTips{cm}{}  
\newdir{ >}{{}*!/-8pt/@{>}}
\newdir{|>}{%
	!/1.6pt/@{|}*:(1,-.2)@^{>}*:(1,+.2)@_{>}}

\usepackage{todonotes}

\usepackage{mathtools}
\usepackage{hyperref}
\usepackage[capitalize, nameinlink]{cleveref}

\usepackage{aliascnt}
\usepackage[appendix=strip]{apxproof}

\theoremstyle{plain}
\newtheoremrep{theorem}{Theorem}[section]
\Crefname{theorem}{Thm.}{Thm.}
\newaliascnt{propositioncnt}{theorem}
\newtheoremrep{proposition}[propositioncnt]{Proposition}
\Crefname{propositioncnt}{Prop.}{Prop.}
\newaliascnt{lemmacnt}{theorem}
\newtheoremrep{lemma}[lemmacnt]{Lemma}
\Crefname{lemmacnt}{Lem.}{Lem.}
\newaliascnt{corollarycnt}{theorem}
\newtheoremrep{corollary}[corollarycnt]{Corollary}
\Crefname{corollarycnt}{Cor.}{Cor.}
\newaliascnt{factcnt}{theorem}

\Crefname{factcnt}{Fact}{Fact}
\newaliascnt{assumptioncnt}{theorem}

\Crefname{assumptioncnt}{Asm.}{Asm.}
\theoremstyle{definition}
\newaliascnt{remarkcnt}{theorem}
\newtheorem{remark}[remarkcnt]{Remark}
\Crefname{remarkcnt}{Rem.}{Rem.}
\newaliascnt{notationcnt}{theorem}

\Crefname{notationcnt}{Notation}{Notation}

\newaliascnt{requirementcnt}{theorem}

\Crefname{requirementcnt}{Requirement}{Requirement}
\newaliascnt{requirementscnt}{theorem}

\Crefname{requirementscnt}{Requirements}{Requirements}

\Crefname{researchquestioncnt}{RQ}{RQ}

\theoremstyle{definition}
\newaliascnt{definitioncnt}{theorem}
\newtheoremrep{definition}[definitioncnt]{Definition}
\Crefname{definitioncnt}{Def.}{Def.}
\newaliascnt{examplecnt}{theorem}
\newtheoremrep{example}[examplecnt]{Example}
\Crefname{examplecnt}{Example}{Examples}

\crefformat{section}{\S{}#2#1#3}
\Crefname{table}{Table}{Table}
\Crefname{figure}{Fig.}{Fig.}
\Crefname{equation}{}{}
\Crefname{line}{Line}{Line}

\usepackage{stmaryrd,amssymb,amsmath}

\usepackage{mathrsfs}

\usepackage{wrapfig}
\usepackage{comment}
\specialcomment{auxproof}
{\mbox{}\newline\textbf{BEGIN: AUX-PROOF}\dotfill\newline}
{\mbox{}\newline\textbf{END: AUX-PROOF}\dotfill\newline}
\excludecomment{auxproof}

\newcommand{\dHL}{\ensuremath{\mathrm{dFHL}}}
\newcommand{\dHLFB}{\ensuremath{\mathrm{dFHL}^{\downarrow}}}

\newcommand{\xtgt}{y_{\mathrm{tgt}}}
\newcommand{\ytgt}{y_{\mathrm{tgt}}}
\newcommand{\bmax}{b_{\mathrm{max}}}
\newcommand{\bmin}{b_{\mathrm{min}}}
\newcommand{\amax}{a_{\mathrm{max}}}
\newcommand{\vmin}{v_{\mathrm{min}}}
\newcommand{\vmax}{v_{\mathrm{max}}}

\newcommand{\cruise}{\mathtt{cruise}}
\newcommand{\brake}{\mathtt{brake}}

\newcommand{\DM}{DM}
\newcommand{\AC}{AC}

\newcommand{\POV}[1]{\ensuremath{\mathsf{POV}{#1}}}
\newcommand{\SV}{\ensuremath{\mathsf{SV}}}

\newcommand{\aheadSL}{\ensuremath{\mathsf{aheadSL}}}
\newcommand{\behindSL}{\ensuremath{\mathsf{behindSL}}}
\newcommand{\true}{\ensuremath{\mathsf{true}}}
\newcommand{\false}{\ensuremath{\mathsf{false}}}

\newcommand{\BC}{BC}

\newcommand{\dRSS}{\ensuremath{\mathsf{dRSS}}}

\newcommand{\Safe}{\mathsf{Safe}}
\newcommand{\Env}{\mathsf{Env}}
\newcommand{\Goal}{\mathsf{Goal}}

\newcommand{\Variables}{\ensuremath{V}}
\newcommand{\expa}{\ensuremath{e}}

\newcommand{\val}{\ensuremath{v}}
\newcommand{\var}{\ensuremath{x}}
\newcommand{\term}{\ensuremath{e}}
\newcommand{\vars}{\ensuremath{\mathbf{x}}}
\newcommand{\funs}{\ensuremath{\mathbf{f}}}
\newcommand{\invariant}{\ensuremath{e_\mathsf{inv}}}

\newcommand{\variant}{\ensuremath{e_\mathsf{var}}}

\newcommand{\terminator}{\ensuremath{e_\mathsf{ter}}}

\newcommand{\asserta}{\ensuremath{A}}
\newcommand{\assertb}{\ensuremath{B}}
\newcommand{\assertc}{\ensuremath{C}}
\newcommand{\assertd}{\ensuremath{D}}

\newcommand{\subst}[3]{{#1}[{#2} / {#3}]}

\newcommand{\coma}{\alpha}
\newcommand{\comb}{\beta}

\newcommand{\dynamics}[2]{\delta_{#1}^{#2}}

\newcommand{\skipClause}{\mathsf{skip}}
\newcommand{\seqClause}[2]{{#1};{#2}}
\newcommand{\assignClause}[2]{#1 \mathop{{:}{=}} #2}

\newcommand{\dwhileKeyword}{\mathsf{dwhile}}
\newcommand{\dwhileHeader}[1]{\dwhileKeyword\,(#1)}
\newcommand{\dwhileClause}[2]{\dwhileHeader{#1}\,\{#2\}}

\newcommand{\odeClause}[2]{\dot{#1} = #2}

\newcommand{\whileKeyword}{\mathsf{while}}
\newcommand{\whileHeader}[1]{\whileKeyword\,(#1)}
\newcommand{\whileClause}[2]{\whileHeader{#1}\,\{#2\}}

\newcommand{\ifHeader}[1]{\mathsf{if}\,(#1)}
\newcommand{\elseKeyword}{\mathsf{else}}
\newcommand{\ifThenElse}[3]{\ifHeader{#1}\,\{#2\}\,\elseKeyword{}\,\{#3\}}

\newcommand{\limply}{\Rightarrow}
\newcommand{\bigland}{\bigwedge}

\newcommand{\skipcrule}{\textsc{Skip}}
\newcommand{\seqrule}{\textsc{Seq}}
\newcommand{\assignrule}{\textsc{Assign}}
\newcommand{\ifrule}{\textsc{If}}
\newcommand{\whilerule}{\textsc{W}}
\newcommand{\dwhilerule}{\textsc{DW}}

\newcommand{\limprule}{\textsc{LImp}}

\newcommand{\store}{\ensuremath{\rho}}
\newcommand{\update}[3]{\ensuremath{{#1}[{#2} \to {#3}]}}

\newcommand{\sem}[2]{\ensuremath{\left\llbracket {#1} \right\rrbracket_{#2}}}

\newcommand{\red}[1]{\to}

\newcommand{\lieder}[3]{\mathcal{L}_{\odeClause{#1}{#2}}\,#3}

\newcommand{\set}[1]{\left\{ {#1} \right\}}
\newcommand{\setcomp}[2]{\left\{ {#1} \,  \middle| \, {#2} \right\}}

\newcommand{\closure}[1]{cl(#1)}

\newcommand{\N}{\ensuremath{\mathbb{N}}}
\newcommand{\R}{\ensuremath{\mathbb{R}}}
\newcommand{\Var}{\ensuremath{\mathbf{Var}}}

\newcommand{\KeYmaeraX}{\textsc{KeYmaera~X}}  

\definecolor{darkgreen}{rgb}{0,0.5,0}
\newcommand{\improvement}[1]{}
\newcommand{\todoproof}[1]{}

\newcommand{\asLongAsClause}[2]{{#1}\downarrow{#2}}
\newcommand{\switchClause}[3]{{#1}\rightarrow_{#3}{#2}}
\newcommand{\simplexClause}[4]{{#1}\begin{smallmatrix}{}^{#4}\!\leftarrow\\\rightarrow_{#3}\end{smallmatrix}{#2}}
\newcommand{\hquin}[5]{{#1}:\left[{#2}\right]\,{#3}\,\left[{#4}\right]:{#5}}

\newcommand{\trace}{\ensuremath{\sigma}}
\newcommand{\Environment}{\ensuremath{E}}

\newcommand{\restr}[2]{{#1}_{|{#2}}}

\newcommand{\xpov}{\ensuremath{x_{\POV{}}}}
\newcommand{\dxpov}{\ensuremath{\dot{x}_{\POV{}}}}
\newcommand{\vpov}{\ensuremath{v_{\POV{}}}}
\newcommand{\dvpov}{\ensuremath{\dot{v}_{\POV{}}}}
\newcommand{\apov}{\ensuremath{a_{\POV{}}}}

\renewcommand{\xtgt}{x_{\mathrm{tgt}}}

\newcommand{\ifThen}[2]{\ifHeader{#1}\,\{#2\}}
\renewcommand{\limprule}{\ensuremath{\limply}}
\newcommand{\untilrule}{\ensuremath{\downarrow}}
\newcommand{\botrule}{\ensuremath{\bot}}

\DeclareMathOperator{\traceend}{end}

\renewcommand{\closure}[1]{\overline{#1}}

\newcommand{\Nbar}{\ensuremath{\overline{\N}}}
\newcommand{\eqv}{\ensuremath{\sim}}
\newcommand{\abs}[1]{\left|{#1}\right|}

\newcommand{\Logic}{\ensuremath{C}}
\newcommand{\Physical}{\ensuremath{P}}

\newcommand{\concatbin}[2]{\ensuremath{{#1} \cdot {#2}}}
\newcommand{\concatfinsym}{\ensuremath{\odot}}
\newcommand{\concatfin}[3]{\ensuremath{\concatfinsym_{{#2}=0}^{#3} {#1}_{#2}}}
\newcommand{\concatinf}[2]{\ensuremath{\concatfinsym_{{#2}=0}^{+\infty} {#1}_{#2}}}
\newcommand{\tracezero}[1]{\ensuremath{\delta_{#1}}}

\DeclareMathOperator{\dom}{dom}

\DeclareMathOperator{\RSS}{RSS}
\DeclareMathOperator{\GARSS}{GA-RSS}
\DeclareMathOperator{\CARSS}{CA-RSS}
\newcommand{\comalayer}{\ensuremath{\coma_l}}

\renewcommand{\epsilon}{\varepsilon}

\usepackage{tikz}
\usetikzlibrary{calc,fit}

\tikzset{fit margins/.style={/tikz/afit/.cd,#1,
    /tikz/.cd,
    inner xsep=\pgfkeysvalueof{/tikz/afit/left}+\pgfkeysvalueof{/tikz/afit/right},
    inner ysep=\pgfkeysvalueof{/tikz/afit/top}+\pgfkeysvalueof{/tikz/afit/bottom},
    xshift=-\pgfkeysvalueof{/tikz/afit/left}+\pgfkeysvalueof{/tikz/afit/right},
    yshift=-\pgfkeysvalueof{/tikz/afit/bottom}+\pgfkeysvalueof{/tikz/afit/top}},
    afit/.cd,left/.initial=2pt,right/.initial=2pt,bottom/.initial=2pt,top/.initial=2pt}

\newcommand{\newFallbackClauseFull}[5]{%
\begin{tikzpicture}[baseline=(coma.base)]
  \node[fit margins={left=1pt,right=0pt,top=1pt,bottom=1pt},outer sep=0] (coma) {\ensuremath{#1}};
  \node[fit margins={left=0pt,right=1pt,top=1pt,bottom=1pt},outer sep=0,anchor=base] (comb) at ($($(coma.east)!(coma.base)!(coma.south east)$)+(#5,0)$) {\ensuremath{#2}};
  \node[inner sep=0,outer sep=0] (comab) at (coma.east) {};
  \node[inner sep=0,outer sep=0] (combb) at (comb.west) {};
  \draw[->] (comab)
    to node[fit margins={left=0pt,right=0pt,top=0pt,bottom=1pt},outer sep=0,below,anchor=north] {\scriptsize{\ensuremath{#3}}}
       node[fit margins={left=0pt,right=0pt,top=0pt,bottom=1pt},outer sep=0,below,anchor=north,pos=1] {\scriptsize{\ensuremath{#4}}}
    (combb);
\end{tikzpicture}%
}

\newcommand{\newFallbackClauseDefaultLength}{0.8}
\newcommand{\newFallbackClause}[4]{\newFallbackClauseFull{#1}{#2}{#3}{#4}{\newFallbackClauseDefaultLength}}
\newcommand{\newSwitchClause}[4]{\newFallbackClause{#1}{#2}{#3}{#4}}

\newcommand{\newSimplexClauseFull}[6]{%
\begin{tikzpicture}[baseline=(coma.base)]
  \node[fit margins={left=1pt,right=0pt,top=1pt,bottom=1pt},outer sep=0] (coma) {\ensuremath{#1}};
  \node[fit margins={left=0pt,right=1pt,top=1pt,bottom=1pt},outer sep=0,anchor=base] (comb) at ($($(coma.east)!(coma.base)!(coma.south east)$)+(#6,0)$) {\ensuremath{#2}};
  \node[inner sep=0,outer sep=0] (comab) at ($(coma.east)!0.3!(coma.south east)$) {};
  \node[inner sep=0,outer sep=0] (combb) at ($(comb.west)!(comab)!(comb.south west)$) {};
  \node[inner sep=0,outer sep=0] (comat) at ($(coma.east)!0.3!(coma.north east)$) {};
  \node[inner sep=0,outer sep=0] (combt) at ($(comb.west)!(comat)!(comb.north west)$) {};
  \draw[->] (comab)
    to node[fit margins={left=0pt,right=0pt,top=0pt,bottom=1pt},outer sep=0,anchor=north,below] {\scriptsize{\ensuremath{#3}}}
       node[fit margins={left=0pt,right=0pt,top=0pt,bottom=1pt},outer sep=0,anchor=north,below,pos=1] {\scriptsize{\ensuremath{#4}}}
    (combb);
  \draw[->] (combt) to node[fit margins={left=0pt,right=0pt,top=1pt,bottom=0pt},outer sep=0,anchor=south,above] {\scriptsize{\ensuremath{#5}}} (comat);
\end{tikzpicture}%
}

\newcommand{\newSimplexClauseDefaultLength}{0.8}
\newcommand{\newSimplexClause}[5]{\newSimplexClauseFull{#1}{#2}{#3}{#4}{#5}{\newSimplexClauseDefaultLength}}

\title{\LARGE \bf
Formal Verification of Safety Architectures for Automated Driving
}

\author{\IEEEauthorblockN{
Clovis Eberhart\IEEEauthorrefmark{1}\IEEEauthorrefmark{3},
J\'er\'emy Dubut\IEEEauthorrefmark{2}\IEEEauthorrefmark{1}, 
James Haydon\IEEEauthorrefmark{1},
and
Ichiro Hasuo\IEEEauthorrefmark{1}\IEEEauthorrefmark{4}
\thanks{The authors are supported by JST ERATO HASUO Metamathematics for Systems Design Project (No. JPMJER1603) and JST START Grant (No. JPMJST2213).
IH is supported by JST CREST Grant (No. JPMJCR2012).}
}
\IEEEauthorblockA{\IEEEauthorrefmark{1}National Institute of Informatics, Tokyo, Japan}
\IEEEauthorblockA{\IEEEauthorrefmark{2}National Institute of Advanced Industrial Science and Technology (AIST), Tokyo, Japan}
\IEEEauthorblockA{\IEEEauthorrefmark{3}Japanese-French Laboratory for Informatics, IRL3527, Tokyo, Japan}
\IEEEauthorblockA{\IEEEauthorrefmark{4}The Graduate University for Advanced Studies (SOKENDAI), Hayama, Japan}}

\begin{document}

\maketitle
\thispagestyle{empty}
\pagestyle{empty}

\begin{abstract}
\emph{Safety architectures} play a crucial role in the
safety assurance of automated driving vehicles (ADVs).
They can be used as \emph{safety envelopes} of black-box ADV
controllers, and for \emph{graceful degradation} from one ODD to
another.
Building on our previous work on the formalization of
\emph{responsibility-sensitive safety (RSS)}, we introduce a novel
program logic that accommodates assume-guarantee reasoning and
fallback-like constructs.
This allows us to formally define and prove the safety of
existing and novel safety architectures.
We apply the logic to a pull over scenario and experimentally evaluate
the resulting safety architecture.
\end{abstract}

\section{Introduction}
\label{sec:intro}

Safety of automated driving vehicles (ADVs) is a problem of growing industrial and social interest. New technologies
\begin{auxproof}
 , 
 especially in sensing and perception (such as lidars and deep neural networks), 
\end{auxproof}
are making ADV technologically feasible; but for social acceptance, their safety should be guaranteed and explained.

\begin{auxproof}
 Many existing approaches to the safety assurance of ADVs are \emph{statistical}, such as accident statistics and testing (typically by computer simulation). Another family of approaches---namely \emph{logical} ones---are attracting growing attention, too.
 In logical approaches, safety is stated as a mathematical theorem and it is given a logical proof. Such approaches via proofs are called \emph{formal verification} and have been actively pursued for  software and computing hardware, with a number of notable successes. 

Principal advantages of formal verification are \emph{strong guarantees} (logical proofs never go wrong) and \emph{explainability} (a proof is a comprehensive record of step-by-step arguments towards a safety claim). Therefore, formal verification are widely seen as an important component in  society's effort towards ADV safety, complementing statistical approaches. 
\end{auxproof}

In this paper, we pursue \emph{formal verification} of ADV safety, that is, to provide its mathematical proofs. This \emph{logical} approach,  compared to \emph{statistical} approaches such as accident statistics and scenario-based testing, offers much stronger guarantees (controllers backed by logical proofs never go wrong).

Moreover, a mathematical proof serves as a detailed record of
\emph{safety arguments}, where 1) each reasoning step is
mathematically verified, and 2) each assumption is explicated.
Thanks to these features, other parties can easily scrutinize those
proofs as safety arguments, making them an important
\emph{communication medium} in society's efforts towards accountable
ADV safety.

\subsection{Safety Architecture}\label{subsec:introSafetyArch}

\begin{wrapfigure}[7]{r}{0pt}
\centering
\includegraphics[width=.24\textwidth]{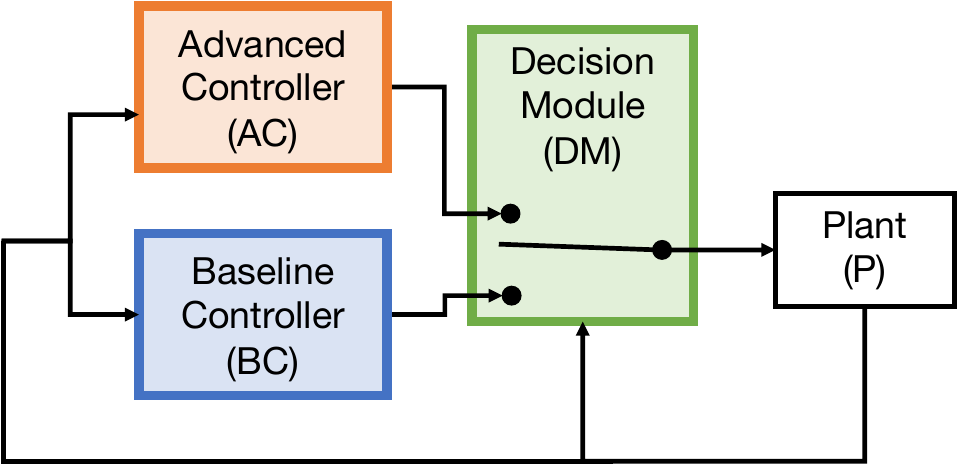}
\caption{simplex architecture}
\label{fig:introSimplex}
\end{wrapfigure}
Formal verification of real-world ADVs, however, is far from straightforward. This is because of the \emph{modeling problem}: for rigorous mathematical proofs,  one needs rigorous \emph{definitions} of all the concepts involved. Such definitions amount to mathematical \emph{modeling} of target systems, which is hard 
for ADVs due to  their complexity.


An effective countermeasure to the modeling problem---advocated e.g.\ in RSS, 
see \cref{subsec:introRSS}---is given by \emph{safety architectures}. An example, called the \emph{simplex architecture}~\cite{CrenshowGRSK07,SetoKSC98}, is shown in \cref{fig:introSimplex}. 
Here,  the \emph{advanced
  controller} (\AC{}) is a given controller (typically black-box); the
\emph{baseline controller} (\BC{}) is a simpler controller which emphasizes safety; and
the \emph{decision module} (\DM{}) switches between the two
controllers.  \DM{}  uses  \AC{} as often as possible. However, when  it finds  the current situation to be safety
critical, it switches to  safety-centric \BC{}.

The simplex architecture (\cref{fig:introSimplex}) exemplifies one application of safety architectures, namely as \emph{safety envelopes}.
Here, \BC{} and \DM{}  together form a safety envelope of \AC{}, taking over the control when needed. 
In particular, a safety proof of the whole system is possible even if \AC{} is a black box---the safety of \BC{} and the plant P, together with the ``contract'' imposed on \AC{} by \DM{}, is enough. This way, we can confine the modeling problem to a black-box \AC{} and conduct formal verification.

\begin{wrapfigure}[7]{r}{0pt}
\centering
\includegraphics[width=.25\textwidth]{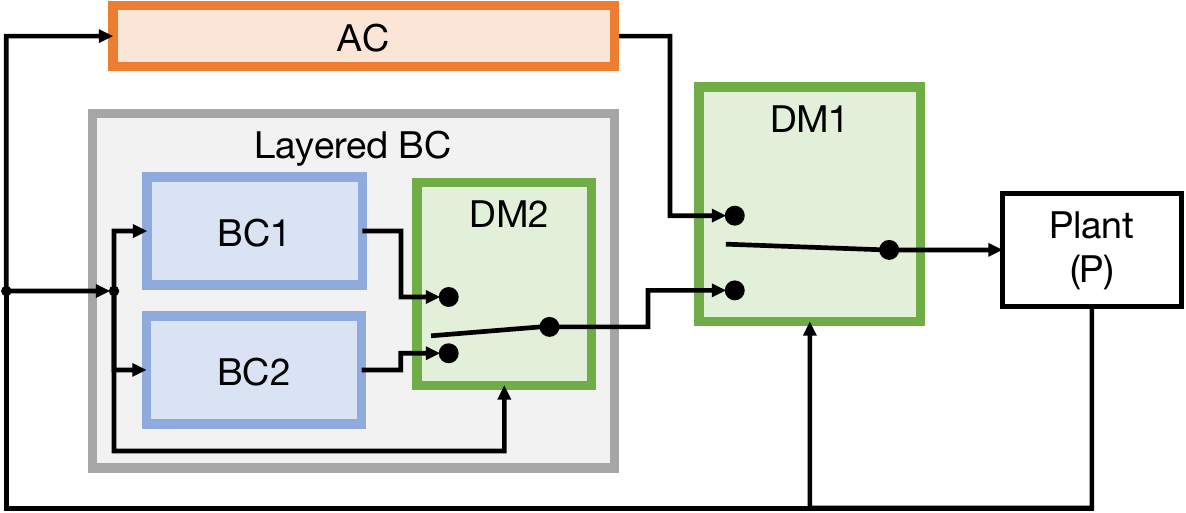}
\caption{layered simplexes}
\label{fig:introLayered}
\end{wrapfigure}
Another important application of safety architectures is  \emph{graceful degradation}, that is, a fallback mechanism to  limited yet guaranteed safety under  hostile environments. \cref{fig:introLayered} shows what we call the \emph{layered simplex architecture}. Here, \BC{}2 and \DM{}2 together form \BC{}1's safety envelope; the composite controller (the \emph{layered \BC{}}) forms a safety envelope of \AC{} with \DM{}1. 
 \BC{}1 and \DM{}1 come with stronger guarantees but require stronger assumptions;  \BC{}2 and \DM{}2, with weaker guarantees and assumptions, realize graceful degradation. Different assumptions imposed by the two can be thought of as different ODDs.


\begin{auxproof}
 The need of graceful degradation is omnipresent in automated driving, where environments come with ample uncertainties and thus assumptions can be violated. Different assumptions required by \BC{}1 and \BC{}2, respectively, can also be understood as different operational design domains (ODDs). 
\end{auxproof}

\subsection{Responsibility-Sensitive Safety (RSS) }\label{subsec:introRSS}

\emph{Responsibility-sensitive safety (RSS)} is a methodology,
proposed in~\cite{ShalevShwartzSS17RSS}, for the formal verification
of ADV safety.
It  circumvents the modeling problem by 1) thinking of each vehicle as
a black box and 2) imposing a contract, called \emph{RSS rules}, on
it.
The methodology---in particular, RSS rules as its central
construct---has many real-world applications, such as attribution of
liability, safety metrics, and regulations and standards.
See e.g.~\cite{Hasuo22_GARSS} for a detailed discussion.

An RSS rule $(P,\alpha)$ is a pair of an \emph{RSS condition} $P$ and a \emph{proper response} $\alpha$ (a specific control strategy). The RSS condition $P$  must ensure the safety of the execution of $\alpha$:
\begin{lemma}[conditional safety] \label{lem:condSafe}
  The execution of $\alpha$, starting in a state where $P$ is true, is collision-free. 
\end{lemma}

 A mathematical proof of this lemma is widely feasible thanks to the simplicity of $P$ and $\alpha$: they do not mention the internal working of ADVs (see below). This is how RSS enables formal verification of ADVs.

\begin{example}[one-way traffic~\cite{ShalevShwartzSS17RSS}]\label{ex:RSSSafetyDistance}
Consider~\cref{fig:onewayTraffic}, where the subject vehicle
  (\SV{}, $\mathrm{car}_\mathrm{rear}$) drives behind another car
  (\POV{}, $\mathrm{car}_\mathrm{front}$). 
The RSS condition 
is
  \begin{math}
    P \;=\;\bigl(x_{f} - x_{r} > \dRSS(v_{f}, v_{r})\bigr)
  \end{math},
  where $\dRSS(v_{f}, v_{r})$ is the \emph{RSS safety distance} 
  \begin{equation}\label{eq:RSSMinDist}\small
    \begin{aligned}
    &  \max\Bigl(\,0,\,
          v_{r}\rho + \frac{1}{2}\amax  \rho^2 + \frac{(v_{r} + \amax  \rho)^2}{2\bmin} -\frac{v_{f}^2}{2\bmax}\,\Bigr).
    \end{aligned}
  \end{equation}
  Here $x_{f}, x_{r}$ are positions of the cars,
  $v_{f}, v_{r}$ are velocities,
$\rho$
  is the  \emph{response time} for $\SV{}$, 
$a_{\max}$ is the maximum
acceleration rate,
  $\bmin$ is the maximum comfortable braking rate, and $\bmax$ is the maximum emergency
  braking rate.

 The proper response $\alpha$  dictates  \SV{}
        to engage the maximum comfortable braking (at rate $\bmin$)
        when condition
$P$
is about to be violated.
 Proving the conditional safety lemma for $(P,\alpha)$ is not hard. See~\cite{ShalevShwartzSS17RSS,Hasuo22RSSarXiv}  (informal) and~\cite{Hasuo22_GARSS} (formal).
\end{example}

  \begin{wrapfigure}[4]{r}{0pt}
 \centering
 \includegraphics[bb=66 217 305 271,clip,width=9em]{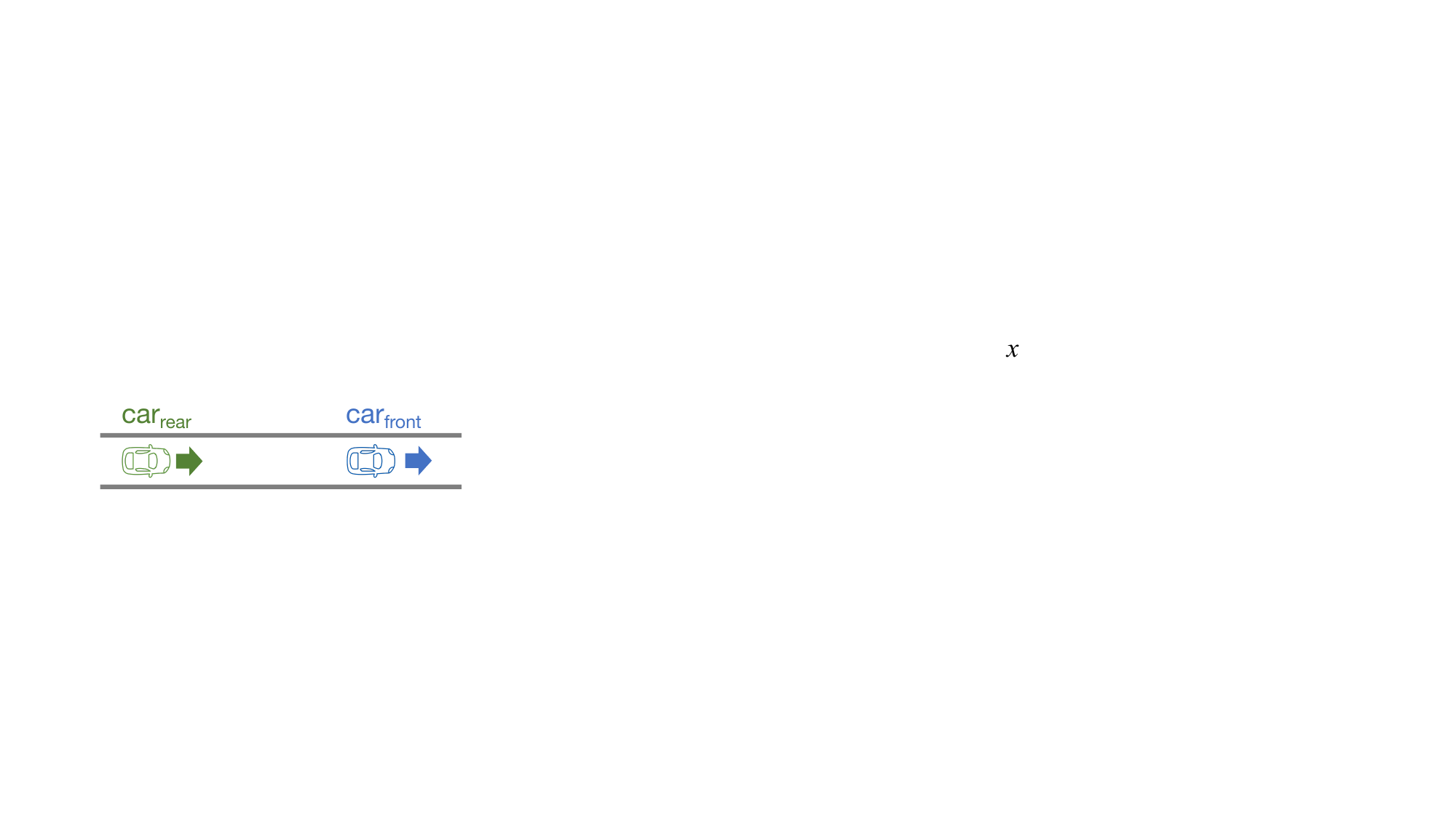}
 \caption{one-way traffic}
 \label{fig:onewayTraffic}
\end{wrapfigure}
\noindent
The current work builds on one important application of RSS rules, namely their use in the simplex architecture (\cref{fig:introSimplex}). 
The conceptual structure of RSS rules maps naturally to the simplex architecture: \AC{} is an ADV; \BC{} executes a proper response $\alpha$; and  \DM{} switches to \BC{} if the RSS condition $P$ is violated, switching back to \AC{} when $P$ is robustly satisfied.

\subsection{Logical Formalization of RSS by the Program Logic $\dHL$}
 An RSS rule must be derived
for each individual driving scenario.  Broad application of RSS requires many such derivations; 
doing so informally (in a pen-and-paper manner) is not desirable for scalability, maintainability, and accountability. 

\begin{wrapfigure}[11]{r}{0pt}
\includegraphics[bb=0 0 337 458,clip,scale=.24]{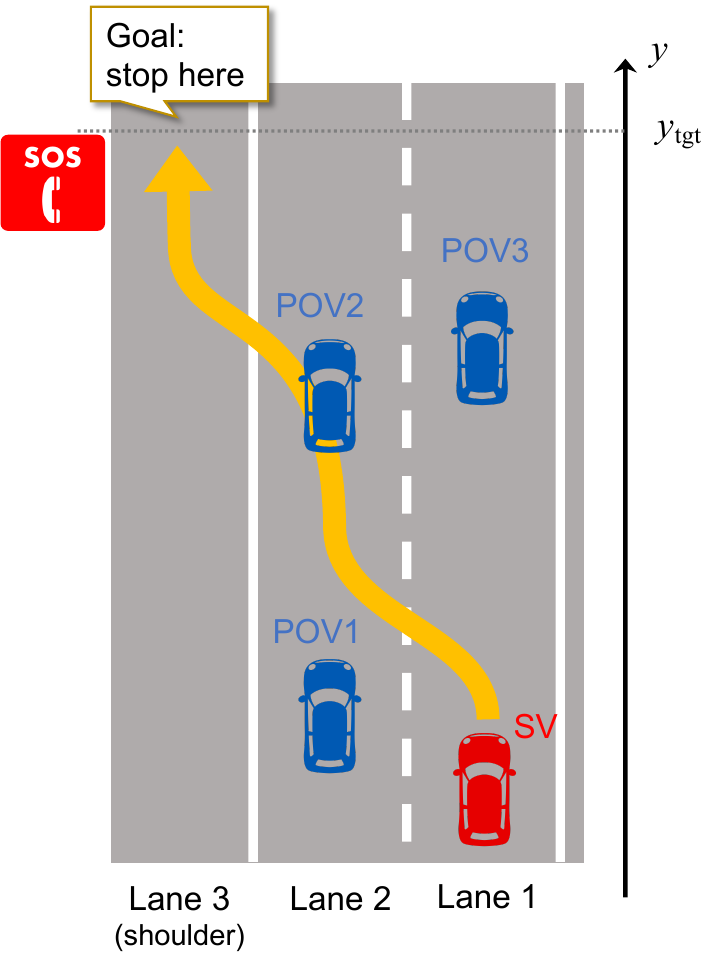}
 \caption{pull over}
 \label{fig:pulloverIntro}
\end{wrapfigure}
This is why we pursued the formalization of RSS in~\cite{Hasuo22_GARSS}. We introduced a logic $\dHL$---a symbolic framework to write proofs in---extending classic \emph{Floyd--Hoare logic}~\cite{Hoare69} with differential equations. The logic $\dHL$ derives \emph{Hoare quadruples} $[P]\,\alpha\,[Q]\!:\!S$;\footnote{In~\cite{Hasuo22_GARSS} we used  delimiters $\{P\}\,\alpha\,\{Q\}:S$ for Hoare quadruples. We use $[\_\,]$ in this paper to emphasize their \emph{total correctness} semantics (\cref{def:hoare:quintuples}).} it means that the execution of a \emph{hybrid program} $\alpha$, started when a \emph{precondition} $P$ is true, terminates, makes a \emph{postcondition} $Q$ true at the end of the execution, and a \emph{safety condition} $S$ true throughout the execution. 

Note that \cref{lem:condSafe} of RSS  naturally corresponds to the validity of  a Hoare quadruple: if we let $P$ be an RSS condition and $\alpha$ be a proper response, then $S$ (expressing collision-freedom) is ensured throughout. Moreover, we can use the postcondition $Q$ to express the \emph{goal} of $\alpha$, such as to stop at a desired position. This extension of RSS, where RSS rules can guarantee not only safety but also goal achievement, is called \emph{GA-RSS}  in~\cite{Hasuo22_GARSS}. 
To distinguish from $\GARSS$, we denote classical RSS by $\CARSS$
for \emph{collision-avoiding RSS}.

Another major benefit of  $\dHL$ is \emph{compositional} reasoning.
\begin{auxproof}
 Exploiting the compositionality of program logic in general (which is also the case of $\dHL{}$),
\end{auxproof}
We devised in~\cite{Hasuo22_GARSS} a workflow in which a complex scenario is split into simpler subscenarios, and RSS rules are derived in a divide-and-conquer manner.

As a case study, in~\cite{Hasuo22_GARSS}, we derived an RSS rule for the \emph{pull over} scenario, shown in \cref{fig:pulloverIntro}. Thus we extended the application domain of RSS to such complex scenarios.


\subsection{This Work: Formal Verification of Safety Architectures}
In this paper, we extend 
$\dHL{}$~\cite{Hasuo22_GARSS} and introduce the logic
$\dHLFB$ called \emph{differential Floyd--Hoare logic with interruptions},
for the purpose of proving that safety architectures are indeed safe.
Using $\dHLFB$, we  address the following questions.
  (On safety envelopes) Let  $(P,\alpha)$ be an RSS rule that is safe (\cref{lem:condSafe}). Can we prove that the simplex architecture, using $P,\alpha$ as \DM{} and \BC{}, is indeed safe?
  (On graceful degradation) Let  $(P_{1},\alpha_{1})$ and $
    (P_{2},\alpha_{2})$ be RSS rules. How exactly should we use them to form the layered simplex architecture (\cref{fig:introLayered})? 
    What safety guarantee is provided under what assumption? 
    Can we give mathematical proofs for such guarantees that are \emph{compositional}, that is,  can they be easily obtained by combining proofs of \cref{lem:condSafe} for the two RSS rules?

The new logic $\dHLFB$ has the following major departures from $\dHL$. Firstly, $\dHLFB$ derives \emph{Hoare quintuples}
\begin{equation}\label{eq:HoareQuintuple}
  \hquin{A}{P}{\coma}{Q}{G}.
\end{equation}
The components $A, G$ are called an \emph{assumption} and a \emph{guarantee}, respectively, and accommodate assume-guarantee type reasoning typical in safety architectures. Comparing to Hoare quadruples $[P]\,\alpha\,[Q]:S$ in~\cite{Hasuo22_GARSS},  the safety condition $S$ (that must hold throughout $\alpha$'s execution) is split into \emph{an assumption throughout} ($A$) and \emph{a guarantee throughout} ($G$).

Secondly, as part of \emph{hybrid programs} that we use to model
driving situations in $\dHLFB$, we introduce the construct
$\asLongAsClause{\alpha}{A}$ (``$\alpha$ \emph{as long as} $A$'');
this executes the program $\alpha$  while the condition $A$ is true,
and halts otherwise.
This construct, introduced as a suitable syntactic sugar
(\cref{def:asLongAs}), turns out to be expressive enough for the
safety architectures of our interest.
In particular, the following  constructs can be expressed:
$\newFallbackClause{\coma}{\comb}{A}{B}$ (the \emph{fallback} of
$\alpha$ on $\beta$, \cref{def:fallback}) and
$\newSimplexClause{\coma}{\comb}{A}{B}{C}$ (the \emph{simplex} of
$\coma$ and $\comb$ with switching by $A,B,C$,
\cref{rem:simplexByAsLongAs}).

Thirdly, we develop a novel semantical foundation of Hoare quintuples
which, unlike in~\cite{Hasuo22_GARSS}, requires explicit
modeling of continuous dynamics (needed for accommodating assumptions
$A$).
This allows us to formulate \emph{derivation rules} for Hoare
quintuples regarding $\asLongAsClause{\alpha}{A}$,
$\newFallbackClause{\coma}{\comb}{\asserta}{\assertb}$, and
$\newSimplexClause{\coma}{\comb}{C}{D}{C'}$.
These rules come in \emph{strong} and \emph{weak} versions: the strong
is used when the original assumption holds throughout and thus no
fallback is needed (i.e.\ under \emph{stronger} assumptions); the weak
one addresses the other cases (i.e.\ under \emph{weaker} assumptions).

Our main case study is about  a safety envelope with graceful degradation for the pull over scenario (\cref{fig:pulloverIntro}). It uses two RSS rules. The RSS rule we derived in~\cite{Hasuo22_GARSS}---called the (goal-aware) \emph{GA-RSS} rule---guarantees safety and goal achievement (i.e.\  reaching the stopping position), but it comes under the constant-speed assumption on  principal other vehicles (\POV{}s). In case \POV{}s change their speed, we use the RSS rule from~\cite{ShalevShwartzSS17RSS} (the (collision-avoiding) \emph{CA-RSS rule}, \cref{ex:RSSSafetyDistance}), giving up the  goal-achievement guarantee. We present the layered simplex architecture (\cref{fig:introLayered}) that combines the two rules, and prove its \emph{strong} and \emph{weak} guarantees.
The proof is compositional, using the guarantees of the two rules. We present the implementation of the layered simplexes; we show that it  ensures safety, and achieves the goal when possible.

\subsection{Contributions}
We provide a theoretical framework for proving that safety
architectures are indeed safe, emphasizing their application to safety envelopes and
graceful degradation.
Technically:
 1) We extend  $\dHL$~\cite{Hasuo22_GARSS} for RSS~\cite{ShalevShwartzSS17RSS} and introduce a logic $\dHLFB$.  It accommodates assume-guarantee  reasoning by \emph{Hoare quintuples} $A:[P]\,\alpha\,[Q]:G$. 
 2) We introduce the program construct $\asLongAsClause{\alpha}{A}$
(``$\alpha$ \emph{as long as} $A$''), from which fallbacks
$\newSwitchClause{\coma}{\comb}{\asserta}{\assertb}$ and simplexes
$\newSimplexClause{\coma}{\comb}{C}{D}{C'}$ can be expressed.
 3) We introduce derivation rules for those constructs. We prove them sound with respect to their rigorous semantics, which is also our contribution.
 4) As a case study, we present the layered simplex architecture (\cref{fig:introLayered}) for the pull over scenario (\cref{fig:pulloverIntro}), for which we prove ``strong'' and ``weak'' guarantees. 
 5) We implemented, and experimentally evaluated, the layered simplexes.
%


\subsection{Related and Future Work}\label{subsec:relatedWork}

\begin{auxproof}
 Some RSS rules have been implemented and are offered as a library~\cite{GassmannOBLYEAA19IV}. Integration of the goal-aware RSS rules we derive in this paper, in the library, is future work. One advantage of doing so is that the GA-RSS rules will then accommodate varying road shapes.

 Need to cite~\cite{KarimiD22}
 \begin{itemize}
 \item about test-case generation
 \item novelty is an iterative generation of increasingly challenging test cases
 \item demonstrated that auto-pilot-plus-RSS can exhibit unsafe behaviors. This is because, the RSS rules implemented in a safety architecture is essentially for the one-way traffic scenario (\cref{ex:RSSSafetyDistance}) and does not fully take  the studied intersection scenario into account
 \end{itemize}
\end{auxproof}

Using RSS in safety architectures is advocated in the literature, e.g.~\cite{OborilS21}. Ours is the first to formally prove their safety.

There are some extensions of RSS. They study 1) parameter selection for balancing safety and progress~\cite{KonigshofOSS22}, 2) an empirical safeguard layer in addition to RSS~\cite{OborilS21},
\begin{auxproof}
  invoked in case parameter values are too liberal~\cite{OborilS21} (this layer is empirical and does not come with logical safety proofs), 
\end{auxproof}
3) extension to unstructured traffic~\cite{PaschOGS21},
 and 4)    swerves as evasive maneuvers~\cite{deIacoSC20IV}.
Among these, the additional layer approach in~\cite{OborilS21} is the
closest to the current work: it studies the question ``what if the
original assumption is violated,'' as we do for graceful degradation.
A big difference, however, is that our layered simplexes are formally
verified while the empirical extra layer in~\cite{OborilS21} is not.

\begin{auxproof}
 Need to cite~\cite{KonigshofOSS22}
 \begin{itemize}
 \item studies parameter values for RSS rules, such as the maximum acceleration rate, in order to balance safety and progress
 \item focus on lateral safety (unlike the longitudinal safety e.g.\ in \cref{ex:RSSSafetyDistance})
 \item The presented technique is empirical
 \end{itemize}

 \cite{OborilS21}
 \begin{itemize}
 \item $\text{RSS}^{+}$: a proactive extension of RSS
 \item takes a proactive measure in case the situation goes beyond the assumptions for the RSS rules---typically because of too strong assumptions on parameter values 
 \item Its essence is an additional (statistical) layer on top of the RSS-based safety architecture. The proper response suggested by this additional layer, however, does not come with a mathematical safety guarantee, unlike the proper response of the original RSS
 \item (Our hierarchical RSS work is the same in that it adds another layer, but a big difference is that our additional layer is logical again, and the construction is compositional)
 \end{itemize}

 \cite{PaschOGS21}
 \begin{itemize}
 \item Expands the application domain of RSS from structured driving scenarios to unstructured ones with vulnerable road users (VRU)
 \item The main technical focus is the choice of kinematic models for different class of VRUs, and their parameter selection
 \item This can be integrated with the current work
 \end{itemize}

 Recent extensions of RSS include a risk-aware one~\cite{OborilS20IV}  and one that allows swerves as evasive maneuvers~\cite{deIacoSC20IV}. These extensions shall be pursued in our current goal-aware framework. In particular, allowing swerves should be possible, and it will significantly improve the progress of a RSS-supervised controller.
\end{auxproof}


\begin{auxproof}
 Regarding logical approaches to ADV safety, some recent work formalize traffic rules in temporal logic, so that they can be effectively monitored and enforced~\cite{DBLP:conf/ivs/MaierhoferMA22,DBLP:conf/ivs/LinA22}. Major differences from the current work are 1) their rules may not be formally verified while we insist ours are, and 2) we directly implement RSS rules in safety architectures, while the implementation is largely left open in~\cite{DBLP:conf/ivs/MaierhoferMA22,DBLP:conf/ivs/LinA22}. 

\end{auxproof}

\begin{auxproof}
 
 Need to cite~\cite{DBLP:conf/ivs/MaierhoferMA22}, ``Formalization of Intersection Traffic Rules in Temporal Logic''. 
 \begin{itemize}
 \item Formalizing existing traffic rules in temporal logic. It is not about proving the safety of some control strategy.
 \end{itemize}

 Need to cite~\cite{DBLP:conf/ivs/LinA22}, ``Rule-Compliant Trajectory Repairing using Satisfiability Modulo Theories.'' Sounds like our use of RSS rules in the safety architecture. Say how ours differs from theirs. As I understand, their ``rule'' is a traffic rule such as ``give way to pedestrians.''
 \begin{itemize}
 \item Rules are given
 \item They discuss how to repair trajectories
 \item In contrast, ours is about designing rules and proving their safety
 \end{itemize}
\end{auxproof}


\begin{auxproof}
A rule-based decision making system for intersection scenarios is proposed in~\cite{AksjonovK21}, but it does not come with safety proofs. From the RSS point of view, the work~\cite{AksjonovK21}  is suggesting proper responses, for which safety-guaranteeing RSS conditions can be derived using our framework. Doing so is future work.
 \cite{AksjonovK21}
 \begin{itemize}
 \item Proposes a rule-based decision making system for intersection scenarios
 \item No logical formalization, no safety proofs
 \item In relation to the current work, it can be seen as a suggestion of a class of proper responses
 \end{itemize}
\end{auxproof}


Formal verification of ADV safety is pursued in e.g.~\cite{RizaldiISA18,RoohiKWSL18arxiv}. These works adopt much more fine-grained modeling compared to the current work and~\cite{Hasuo22_GARSS}; a price for doing so seems to be scalability and flexibility.

We can easily speculate that the application domain of the current work does not restrict to ADVs. In fact, for our previous work~\cite{Hasuo22_GARSS} already, there have been interests from robotics and the drone industry. We will pursue such broader applications.

\begin{auxproof}
 \begin{itemize}
 \item 
 many optimization- and learning-based planning algorithms for safe driving, such as~\cite{mcnaughton2011motion} (they do not offer rigorous safety guarantee),
 \item 
 testing-based approaches for ADS safety, such as~\cite{LuoZAJZIWX21ASE} (they do not offer rigorous safety guarantee, either), and
 \item 
 runtime verification approaches for ADS safety by reachability analysis, such as~\cite{LiuPA20IV,PekA18IROS}.
 \end{itemize}

 The problem of formally verifying correctness of RSS rules is formulated and investigated in~\cite{RoohiKWSL18arxiv}. Their formulation is based on a rigorous notion of signal; they argue that none of the existing \emph{automated} verification tools is suited for the verification problem. This concurs with our experience so far---in particular, formal treatment of other participants' responsibilities (in the RSS sense) seems to require human intervention. At the same time, in our preliminary manual verification experience in \KeYmaeraX, we see a lot of automation opportunities. Developing proof tactics dedicated to those will ease manual verification efforts.

 Formal (logical, deductive) verification of ADS safety is also pursued in~\cite{RizaldiISA18} using the interactive theorem prover Isabelle/HOL~\cite{NipkowPW02}. The work uses a white-box model of a controller, and a controller must be very simple. This is unlike RSS and the current work, which allows black-box \AC{}s and thus accommodates various real-world controllers such as sampling-based path planners (\cref{subsec:introRSSSafetyArchitecture}).

 In the presence of perceptual uncertainties (such as  errors in position measurement and object recognition), it becomes harder for  \BC{}s and \DM{}s to ensure safety. Making  \BC{}s tolerant of perceptual uncertainties is pursued in~\cite{SalayCEASW20PURSS,DBLP:conf/nfm/KobayashiSHCIK21}. One way to adapt \DM{}s is to enrich their input so that they can better detect  potential hazards. Feeding DNNs' confidence scores is proposed in~\cite{AngusCS19arxiv};  in~\cite{ChowRWGJALKC20}, it is proposed for \DM{}s to look at inconsistencies between perceptual data of different modes.


 Need to cite~\cite{BannourNC21}
 \begin{itemize}
 \item No need to cite
 \item ``Logical'' in the title refers to a level of scenarios (functional---logical---concrete)
 \item their method relies on symbolic automata and constraint solving
 \end{itemize}

 \cite{LauerS22}
 \begin{itemize}
 \item No need to cite
 \item It says ``verification'' but is about test scenario generation
 \end{itemize}
\end{auxproof}

\section{Hybrid Programs}
\label{sec:prog}

In this section, we define the syntax of $\dHLFB$ (\emph{differential
Floyd-Hoare logic with interruptions}), which is similar to
that of $\dHL$~\cite{Hasuo22_GARSS}.
We then define its semantics in terms of valid traces, which
correspond to program executions.

\subsection{Syntax}
\label{sec:prog:syntax}

Definitions
up to
\cref{def:program} mostly come from~\cite{Hasuo22_GARSS},
but they are refined to address environmental nondeterminism.
We highlight the differences from~\cite{Hasuo22_GARSS} when
introducing the notions.

\begin{definition}[terms, assertions]
  \label{def:assertion}
  A \emph{term} is a rational polynomial on a fixed infinite set
  $\Variables$ of variables.
  \emph{Assertions} are generated by the grammar
\begin{math}
   A,B\; ::=\; \true \mid\false\mid e \sim f \mid A \land B \mid A \lor B \mid \lnot A \mid A \limply B
\end{math},
where $e$, $f$ are terms and
  $\sim \ \in \set{=, \leq, <, \neq}$.
\end{definition}

An assertion can be \emph{open} or \emph{closed} (or both, or none).
The exact definition can be found
in~\cite{Hasuo22_GARSS} and the only thing relevant here is
that open assertions describe open subsets of $\R^\Variables$.

 We stick to polynomial terms for simplicity of presentation,
but it is possible to extend the syntax of terms
(see~\cite{Hasuo22_GARSS}).

\begin{definition}[programs]\label{def:program}
  We assume that the set of variables $\Variables = \Variables_\Logic
  \sqcup \Variables_\Physical \sqcup \Variables_\Environment$ is the
  disjoint union of
  $\Variables_\Logic$ (\emph{cyber variables}),
  $\Variables_\Physical$ (\emph{physical variables}), 
  and 
  $\Variables_\Environment$  (\emph{environment variables}).

  \emph{Hybrid program} (or \emph{programs}) are generated by the
  grammar
  \begin{align*}
    \coma, \comb &::=
      \skipClause{} \mid
      \assignClause{\var}{\term} \mid
      \coma;\comb \mid
      \ifThenElse{\asserta}{\coma}{\comb} \mid \\
    & \phantom{\,::=\ } \dwhileClause{\asserta}{\odeClause{\vars}{\funs}}
      \mid \whileClause{\asserta}{\coma}
  \end{align*}
  In $\assignClause{\var}{\term}$, $\var$ is a cyber variable.
  In $\dwhileClause{\asserta}{\odeClause{\vars}{\funs}}$, $\vars$ and
  $\funs$ are lists of the same length, respectively of (distinct)
  physical variables and terms, and we require that $\asserta$ is open. We abbreviate $\ifThenElse{A}{\coma}{\skipClause{}}$ as
$\ifThen{A}{\coma}$.

\end{definition}

The decomposition of $\Variables$ into three sets differs
from~\cite{Hasuo22_GARSS} and allows a better
characterization of the nondeterminism that can happen in a program:
physical variables can only change continuously, while cyber variables
can also change discontinuously, and environment variables can change
even more drastically (see $\dwhileKeyword$ in
\cref{sec:prog:semantics} for more details).

The following program construct is central to our formal analysis
of safety architectures.
The program $\asLongAsClause{\alpha}{A}$ starts the execution of
$\alpha$ and continues \emph{as long as} the assertion $A$ is true.
The execution is interrupted (even if it is not yet completed) once
$A$ is violated.
Note that this is different from $\whileClause{A}{\alpha}$ where
$\alpha$ is repeatedly executed and the check of $A$ is done between
these executions.

\begin{definition}[$\asLongAsClause{\coma}{A}$, ``$\alpha$ as long as $A$'']\label{def:asLongAs}
  Let $A$ be an open assertion. For each program $\alpha$, we  recursively define the program 
  $\asLongAsClause{\coma}{A}$ as the following syntactic sugar.
\begin{displaymath}\small
   \begin{aligned}
    \asLongAsClause{\skipClause{}}{A} \;\equiv\; \skipClause{} ~~~&~~~
    \asLongAsClause{(\assignClause{\var}{\term})}{A} \;\equiv\;
      \ifThen{A}{\assignClause{\var}{\term}} \\
    \asLongAsClause{(\coma;\comb)}{A} &\;\equiv\;
      \ifThen{A}{\asLongAsClause{\coma}{A};\;
      \ifThen{A}{\asLongAsClause{\comb}{A}}} \\
    \asLongAsClause{\ifThenElse{C}{\coma}{\comb}}{A} &\;\equiv\;
      \ifThenElse{C}{\asLongAsClause{\coma}{A}}{\asLongAsClause{\comb}{A}} \\
    \asLongAsClause{\dwhileClause{C}{\odeClause{\vars}{\funs}}}{A} &\;\equiv\;
      \dwhileClause{A \land C}{\odeClause{\vars}{\funs}} \\
    \asLongAsClause{\whileClause{C}{\coma}}{A} &\;\equiv\;
      \whileClause{A \land C}{\asLongAsClause{\coma}{A}} 
  \end{aligned}
\end{displaymath}
\end{definition}

The following program is a basic brick we use to define some safety
architectures.
It executes $\coma$ as long as $C$ is true.
If $C$ holds all the time, then $\coma$ is executed to completion, and
the program stops if $D$ holds.
If $D$ does not hold at the end of execution, or $C$ is violated and
the execution of $\coma$ is interrupted, and $\comb$ starts executing
as a fallback.

\begin{definition}[$\newSwitchClause{\coma}{\comb}{\asserta}{\assertb}$, fallback]
\label{def:fallback}
  Let $A$ be an open assertion. The program $\newSwitchClause{\coma}{\comb}{\asserta}{\assertb}$, called the \emph{fallback} of $\alpha$ on $\beta$ if not $A$,  is introduced as the following syntactic sugar:
\begin{displaymath}
 \newSwitchClause{\coma}{\comb}{\asserta}{\assertb}
 \;\equiv\;
 \seqClause{(\asLongAsClause{\coma}{\asserta})}{\;\ifThen{\neg
  (\asserta \land \assertb)}{\comb}}.
\end{displaymath}
We write $\newFallbackClause{\coma}{\comb}{\asserta}{}$ for
$\newFallbackClause{\coma}{\comb}{\asserta}{\assertb}$ when $\assertb
\equiv \top$.


\end{definition}

\subsection{Semantics}
\label{sec:prog:semantics}

In this section, we define a semantics for $\dHLFB$.
Contrary to~\cite{Hasuo22_GARSS}, which describes a
small-step reduction semantics, here we describe a program's semantics
in terms of its traces, in the style of LTL~\cite{pnueli1977temporal},
which is needed for this new semantics.


\vspace*{-2pt}
\begin{definition}[store]\label{def:store}
  A \emph{store} is a function $\rho\colon V\to \R$ from variables to reals.
  \emph{Store update} $\update{\store}{\var}{\val}$;
  it maps $\var$ to $\val$ and any other variable $\var'$ to
  $\store(\var')$.
  The \emph{value} $\sem{\term}{\store}$ of a term $\term$ in a store
  $\store$ is a real defined as usual by induction on $e$ (see for
  example~\cite[Section~2.2]{Winskel93}).
  The \emph{satisfaction} relation  $\store \vDash A$ between stores
and
  $\dHL$ assertions
is also defined as usual
  (see~\cite[Section~2.3]{Winskel93}).
  We write $\store \eqv \store'$ when $\forall x \in \Variables_\Logic
  \sqcup \Variables_\Physical$, $\store(x) = \store'(x)$.
  %
\end{definition}

\begin{definition}[trace]\label{def:trace}
  A \emph{trace} is a (finite or infinite) sequence
 $\sigma = \bigl(\,(t_{0}, h_{0}), (t_{1}, h_{1}), \dotsc\,\bigr)
$ of pairs, where
 $t_i \in \R_{\geq 0}$ and $h_i
  \colon [0,t_i] \to \R^{\Variables}$ is a continuous function.
  If $\sigma$ above is a sequence of length $n \in \Nbar = \N \cup \set{+\infty}$, we write
  $\dom(\sigma) \equiv \setcomp{(i,t)}{i < n+1, t \leq t_i}$,
  $\sigma(i) = h_i$ and $\sigma(i,t) = h_i(t)$ for $(i,t) \in
  \dom(\sigma)$.
  Given $\sigma$ of finite length $n$, we
  define the \emph{ending state} $\traceend(\sigma) \in
  \R^{\Variables}$ as $\sigma(n,t_n)$.
  Given an assertion $C$, we
  define $\sigma \vDash C$ as for all $(i,t) \in \dom(\sigma)$,
  $\sigma(i,t) \vDash C$.
  We denote by $\tracezero{\store}$ the trace $\bigl((0, f_\store)
  \bigr)$, where $f_\store(0) = \store$.
  Given $\sigma$ as above and $(i,t) \in \dom(\sigma)$, 
  we define
  $\restr{\sigma}{i,t} = \bigl( (t_0,h_0), \ldots, (t_{i-1},h_{i-1}),
  (t,\restr{h_i}{[0,t]}) \bigr)$.
  The concatenation of a finite trace $\sigma$ and a trace $\sigma'$
  is 
 $\concatbin{\sigma}{\sigma'}$.
  Similarly, $\concatfin{\sigma}{i}{n}$ is the concatenation of
  traces $\sigma_0$, $\ldots$, $\sigma_n$, where all are finite
  (except maybe $\sigma_n$), and $\concatinf{\sigma}{i}$ the
  concatenation of finite $\sigma_0$, $\sigma_1$, $\ldots$
\end{definition}



The following definition of valid traces of a program departs from the
 standard definition~\cite{Hasuo22_GARSS,Platzer18}.
We need this change to accommodate environmental assumptions; see
\cref{sec:hoare}.
The main difference is that our traces record all the continuous
dynamics, instead of recording a discrete set of stores that occur
within.

The basic intuition is that the traces of a program $\coma$ can be
thought of as traces of $\coma$ in the traditional sense, but where
environment variables can be changed at any time during
$\dwhileKeyword$s by an unspecified program (the environment).

\begin{definition}[trace semantics]
  \label{def:semantics}
  We say a trace  $\sigma$ is \emph{valid} for a program $\coma$ from store $\store$, denoted
  $\store, \sigma \vDash \coma$, if the following holds (by induction on $\coma$):
  \begin{itemize}
    \item $\store, \tracezero{\store} \vDash \skipClause{}$ and 
    $\store, \tracezero{\update{\store}{\var}{\sem{\term}{\store}}} \vDash
      \assignClause{\var}{\term}$.
    \item $\store, \sigma \vDash \coma;\comb$ iff either $\sigma$ is
      infinite and $\store,\sigma \vDash \coma$; or $\sigma = \sigma_1
      \cdot \sigma_2$ where $\store, \sigma_1 \vDash \coma$ and
      $\traceend(\sigma_1), \sigma_2 \vDash \comb$.
    \item $\store, \sigma \vDash \ifThenElse{C}{\coma}{\comb}$ iff
      either $\store \vDash C$ and $\store, \sigma \vDash \coma$; or
      $\store \nvDash C$ and $\store, \sigma \vDash \comb$.
    \item $\store, \sigma \vDash
      \dwhileClause{C}{\odeClause{\vars}{\funs}}$ iff either $\store
      \nvDash C$ and $\sigma = \tracezero{\store}$; or $\store \vDash
      C$ and there exists $n \in \Nbar$ such that $\sigma =
      \bigl((t_i,h_i)\bigr)_{i<n}$, and
        1) for all $i < n$, 
          $\restr{h_i}{\vars}$ is differentiable, its derivative
          is $\restr{h_i'}{\vars}(t) = \sem{\funs}{h_i(t)}$ for
          all $0 \leq t \leq t_i$, $h_i(x)$ is constant for other
          $x \in \Variables_\Logic \cup \Variables_\Physical$,and $h_i(0) \eqv h_{i-1}(t_{i-1})$,
        2) for all $(i,t) \in \dom{(\sigma)}$, if $(i,t) <
          (n-1,t_{n-1})$, then $h_{i}(t) \vDash C$ (if $n \neq +\infty$),
        and 3) $h_{n-1}(t_{n-1}) \nvDash C$ (if $n \neq +\infty$),
      where $h_{-1}(t_{-1}) = \rho$, and $+\infty - 1 = +\infty$.
    \item $\store, \sigma \vDash \whileClause{C}{\coma}$ iff either
      $\store \nvDash C$ and $\sigma = \tracezero{\store}$; or $\store
      \vDash C$ and there exists $n \in \Nbar$ such that 1) $\sigma =
      \concatfin{\sigma}{i}{n}$, 2)
      for all $i < n$, $h_{i-1}(t_{i-1}), \sigma_i \vDash \coma$
          and $h_i(t_i) \vDash C$, 
        and  3) for $n \neq +\infty$, $h_n(t_n) \nvDash C$.
  \end{itemize}
\end{definition}

\begin{definition}[incomplete traces]
  An infinite trace of the form $((t_i, h_i))_{i \in \N}$ with all
  $t_i > 0$ induces a function $h \colon [0,\sum_i t_i) \to
  \R^{\Variables}$ by for all $\sum_{i \leq n} t_i \leq s < \sum_{i
  \leq n+1}$, $h(s) = h_n(s - \sum_{i \leq n} t_i)$. 
  We say that a valid trace for $\coma$ from $\store$ of the form $\sigma
  \cdot ((t_1, h_1), \ldots, (t_n, h_n))$ with $\sum_i t_i < +\infty$
  is \emph{incomplete} if there is a valid trace of the form $\sigma \cdot
  (t',h') \cdot \sigma'$ where $t' \geq \sum_i t_i$ and for all $s <
  \sum_i t_i$, $h(s) = h'(s)$.
\end{definition}

\begin{auxproof}
  \begin{remark}[determinism]
    In \cref{def:semantics}, for the $\dwhileKeyword$ constructor, if
    $\restr{h}{\Variables_\Environment} \colon \R_{\geq 0} \to
    \R^{\Variables_\Environment}$ is a given function,
    then by the Picard-Lindel\"of theorem, there exists a
    unique maximal solution $h$ to the differential equation.
    Therefore, for a fixed environmental behavior, each program is
    deterministic.
  \end{remark}
\end{auxproof}



  
  



%

\improvement{try to write the semantics that can treat constraints
like $\apov = 0$ (in that case, we also need corresponding Hoare
rules.)}

\begin{example}[one-way traffic, modelling]
  \label{ex:slf-model}
  In this scenario, two vehicles are in the situation described in
  \cref{fig:onewayTraffic}.
  We want to prove that, if the lead vehicle keeps its speed above
  some fixed $\vmin > 0$, then $\SV$ can reach a fixed target $\xtgt$ 
  while avoiding collisions.
  Otherwise, we can still avoid collisions.

  In this scenario, $x$, $\xpov$, $v$, and $\vpov$ are physical
  variables, and $\apov$ is an environment variable.
  The programs of interest are $\coma$ and $\comb$, which are
  defined as follows:
  \begin{displaymath}
    \begin{array}{l}
      \dynamics{\brake}{} \equiv \{\dot{x} = v, \dot{v} = -\bmin,
        \dxpov = \vpov, \dvpov = \apov\},
      \\
      \dynamics{\cruise}{} \equiv \{\dot{x} = v, \dot{v} = 0,
          \dxpov = \vpov, \dvpov = \apov\},
      \\
      \coma \equiv
      \left[\begin{array}{@{}l@{}}
        \dwhileClause{v > \vmin \land x < \xtgt}{\dynamics{\brake}{}} ; \\
        \dwhileClause{x < \xtgt}{\dynamics{\cruise}{}} ; \;
          \dwhileClause{v > 0}{\dynamics{\brake}{}}
      \end{array}\right],
      \\
      \comb \equiv \dwhileClause{v > 0}{\dynamics{\brake}{}}.
    \end{array}
  \end{displaymath}
  When $\POV{}$ keeps its velocity above $\vmin$ at all times, then
  we execute $\coma$.  Otherwise, we execute $\comb$.
  The whole behavior of the system is thus modeled as the fallback
  $\switchClause{\coma}{\comb}{C}$, where $C \equiv (\vpov > \vmin)$.
\end{example}

\section{Hoare Quintuples}
\label{sec:hoare}


\subsection{Correctness}
\label{sec:hoare:correct}
The classic notion of Hoare triples~\cite{Hoare69}---specifying programs' input and output behaviors---has been extended
to Hoare quadruples~\cite{Hasuo22_GARSS} 
to 
incorporate a global safety condition
and continuous dynamics.
 We introduce Hoare quintuples, which further refine Hoare quadruples by
splitting global safety into a global assumption and a
global guarantee.



\begin{definition}[Hoare quintuple]
  \label{def:hoare:quintuples}
  A \emph{Hoare quintuple} $\hquin{A}{P}{\coma}{Q}{G}$ consists of
  assertions $A,P,Q,G$ and a program $\coma$, delimited as shown.
  It is \emph{totally correct} (or simply \emph{correct}) if, for all
  stores $\store \vDash P$ and traces $\sigma = \bigl( (t_i,h_i)
  \bigr)_{i < n}$ valid for $\coma$ from $\store$, then
    1) \emph{termination}: if $\sigma \vDash A$, 
    then $\sigma$ is
      either finite or incomplete,
    2) \emph{postcondition}: if $\sigma$ is finite, then
      $\traceend(\sigma) \vDash Q$,
    3) \emph{safety}: for all $(i,t) \in \dom{(\sigma)}$, 
 	if $\restr{\sigma}{i,t} \vDash A$, then $\restr{\sigma}{i,t}
      \vDash G$.
\end{definition}


This definition slightly differs from that
of~\cite{Hasuo22_GARSS}.
First, since the language is nondeterministic, there are several
possible traces of $\coma$, and we ask that all of them terminate
(except for incomplete traces).
Secondly, it differentiates between global assumptions and global
guarantees.
This makes the roles of $\Env$ and $\Safe$ 
in~\cite[Section~III]{Hasuo22_GARSS} clearer: while they were
both treated as part of the safety condition
in~\cite{Hasuo22_GARSS}, the former is an assumption while
the second is a guarantee.

\begin{example}[one-way traffic, specification]
  \label{ex:slf-spec}
  Here, we want to formally specify the desired behavior of
  \cref{ex:slf-model}.
  We start with the assumptions under which the scenario must run.
  The weaker assumption, which does not require the lead vehicle to
  have a speed above $\vmin$ is $A' \equiv (-\bmin < \apov < \amax
  \land \vpov \geq 0)$.
  The stronger assumption is $A \equiv A' \land (\vpov > \vmin)$.
  The goal under $A'$ is to stop, formalized as $Q' \equiv (v = 0)$;
  and the goal under $A$ is $Q \equiv Q' \land (x \geq \xtgt)$.
  In both cases, the guarantee is to avoid collision, formalized as
  $G \equiv (x < \xpov)$.
  Finally, the essence of $\RSS$ is that it provides a precondition
  under which we can avoid collisions, namely $P \equiv (\xpov - x >
  \dRSS(\vpov, v) \land v > 0)$.
  The Hoare quintuples that we want to show correct are thus:
  \begin{align}
      &\hquin{A}{P}{\newSwitchClause{\coma}{\comb}{C}{}}{Q}{G} &
      &\hquin{A'}{P}{\newSwitchClause{\coma}{\comb}{C}{}}{Q'}{G}\rlap{.}
    \label{eq:example-hoare}
  \end{align}
\end{example}

\subsection{Hoare Rules}
\label{sec:hoare:rules}

We present  logical rules to derive correct Hoare quintuples.
They are listed in~\cref{fig:hoare}.
In such a rule, hypotheses are listed above the horizontal line and
the conclusion below it.
\begin{auxproof}
 For example the ($\skipcrule$) rule states that, if $A \land P \limply
 Q \land G$ is valid (true for all stores), then
 $\hquin{A}{P}{\skipClause}{Q}{G}$ is correct.
\end{auxproof}
\begin{figure*}\footnotesize
  \begin{mathpar}
    \bottomAlignProof
    \AxiomC{$A \land P \limply Q \land G$}
    \RightLabel{(\skipcrule)}
    \UnaryInfC{$\hquin{A}{P}{\skipClause}{Q}{G}$}
    \DisplayProof
    \and
    \bottomAlignProof
    \AxiomC{~}
    \RightLabel{(\botrule)}
    \UnaryInfC{$\hquin{A}{\bot}{\coma}{Q}{G}$}
    \DisplayProof
    \and
    \bottomAlignProof
    \AxiomC{$\begin{array}{c}\hquin{A}{P}{\coma}{Q}{G} \\
    \hquin{A}{Q}{\comb}{R}{G}\end{array}$}
    \RightLabel{(\seqrule)}
    \UnaryInfC{$\hquin{A}{P}{\coma ; \comb}{R}{G}$}
    \DisplayProof
    \and
    \bottomAlignProof
    \AxiomC{$\begin{array}{c}\hquin{A}{P \land C}{\coma}{Q}{G}\\
    \hquin{A}{P \land \neg C}{\comb}{Q}{G}\end{array}$}
    \RightLabel{(\ifrule)}
    \UnaryInfC{$\hquin{A}{P}{\ifThenElse{C}{\coma}{\comb}}{Q}{G}$}
    \DisplayProof
    \and
    \bottomAlignProof
    \AxiomC{ }
    \RightLabel{(\assignrule)}
    \UnaryInfC{$\hquin{A}{\subst{Q}{\expa}{\var}}{\assignClause{\var}{\expa}}{Q}{Q \lor \subst{Q}{\expa}{\var}}$}
    \DisplayProof
    \and
    \bottomAlignProof
    \AxiomC{$\hquin{A}{P \land C \land \variant \gtrsim 0 \land \variant = \var}{\coma}{P \land \variant \gtrsim 0 \land \variant \leq \var - 1}{G}$}
    \RightLabel{(\whilerule)}
    \UnaryInfC{$\hquin{A}{G \land P \land \variant \gtrsim 0}{\whileClause{C}{\coma}}{P \land \neg C \land \variant \gtrsim 0}{G}$}
    \DisplayProof
    \and
    \bottomAlignProof
    \AxiomC{$\begin{array}{rlll}
        \mathsf{inv\colon} & A \land \variant \geq 0 \land \invariant \sim 0 \Rightarrow \lieder{\vars}{\funs}{\invariant} \simeq 0 & \Var(\invariant) \cap \Variables_\Environment = \emptyset \\
        \mathsf{var\colon} & A \land \variant \geq 0 \land \invariant \sim 0 \Rightarrow \lieder{\vars}{\funs}{\variant} \leq \terminator & \Var(\variant) \cap \Variables_\Environment = \emptyset \\
        \mathsf{ter\colon} & A \land \variant \geq 0 \land \invariant \sim 0 \Rightarrow \lieder{\vars}{\funs}{\terminator} \leq 0 & \Var(\terminator) \cap \Variables_\Environment = \emptyset \\
      \end{array}$}
    \RightLabel{(\dwhilerule)}
    \UnaryInfC{$\hquin{A}{\invariant \sim 0 \land \variant \geq 0 \land \terminator < 0}{\dwhileClause{\variant > 0}{\odeClause{\vars}{\funs}}}{\variant = 0}{\invariant \sim 0 \land \variant \geq 0}$}
    \DisplayProof
    \and
    \bottomAlignProof
    \AxiomC{$\begin{array}{cc}
    \multicolumn{2}{c}{\hquin{A}{P}{\coma}{Q}{G}}
    \\
    P' \land A' \land G \limply P &
    G \limply G'\\
    Q \land A' \land G \limply Q' &
    A' \limply A
    \end{array}$}
    \RightLabel{(\limprule)}
    \UnaryInfC{$\hquin{A'}{P'}{\coma}{Q'}{G'}$}
    \DisplayProof
  \end{mathpar}
  \caption{Hoare derivation rules for $\dHLFB$}
  \label{fig:hoare}
\end{figure*}
The rules are similar to those 
in~\cite{Hasuo22_GARSS}.
Besides natural adaptation from quadruples to quintuples, the only 
major change is ($\dwhilerule$), which
includes new hypotheses.
It requires that none of the variables in $\invariant{},\variant{},\terminator{}$ 
are in
$\Variables_\Environment$.
This ensures that changes in environment variables, which are
nondeterministic, cannot change the values of these terms, thus
ensuring that the $\dwhileKeyword$ terminates by the same argument as
in~\cite{Hasuo22_GARSS}.


\begin{theorem}[soundness]
  For all rules in Figure~\ref{fig:hoare}, if the premises are correct,
  then so is the conclusion.
\end{theorem}


\improvement{Hoare rules for closed conditions}

To model safety architectures, we are especially interested in Hoare
rules for the fallback construction $\switchClause{\coma}{\comb}{C}$.
\begin{auxproof}
 In the rest of this section, we describe several Hoare rules for
 fallbacks.
\end{auxproof}

\subsubsection{Combining Guarantees under a Strong Assumption}

\begin{definition}
  \label{def:trace:cut}
  Given a (finite or infinite) trace $\sigma$, such that $\sigma
  \nvDash C$, let $(i,t) \in \dom(\sigma)$ be the smallest index such
  that $\sigma(i,t) \nvDash C$.
  We define the finite trace $\asLongAsClause{\sigma}{C} \equiv
  \restr{\sigma}{i,t}$.
\end{definition}

\todoproof{
\begin{lemma}
If $\sigma \nvDash C$, then $\traceend(\asLongAsClause{\sigma}{C})
\nvDash C$.
\end{lemma}
\begin{lemma}
If $\sigma \nvDash\assertc$, $\store, \sigma \vDash \coma$, and 
$\store \vDash \assertc$, then 
$\traceend(\asLongAsClause{\sigma}{C}) \vDash \closure{\assertc}$.
\end{lemma}
}


\todoproof{
\begin{lemma} $\store, \sigma \vDash \switchClause{\coma}{\comb}{C}$ if and only if
\begin{itemize}
	\item either $\sigma \vDash C$ and $s, \sigma \vDash \coma$,
	\item or there exists $s, \sigma' \vDash \coma$ such that 
		$\asLongAsClause{\sigma}{C} = \asLongAsClause{\sigma'}{C}$ and
		$\traceend(\asLongAsClause{\sigma}{C}), \sigma\uparrow C \vDash \comb$.
\end{itemize}
\end{lemma}
}

\begin{lemma}[strong as-long-as rule]
  \label{lem:aslongas:strong}
  This rule is correct:
  \[
    \AxiomC{$\hquin{A}{P}{\coma}{Q}{G \land C}$}
    \RightLabel{$(\untilrule_s)$}
    \UnaryInfC{$\hquin{A}{P}{\asLongAsClause{\coma}{C}}{Q}{G}$.}
    \DisplayProof
  \]
\end{lemma}


Since $C$ is guaranteed by $\coma$, no trace of $\coma$ under $A$ can
be interrupted.
Therefore, $\asLongAsClause{\coma}{C}$ behaves like $\coma$.

\begin{lemma}[Hoare rule for $\newSwitchClause{\coma}{\comb}{C}{D}$ under
  strong assumption]
  \label{lem:switch-strong}
  If $\hquin{A}{P}{\coma}{Q}{G \land C}$ is correct and $A \land Q
  \limply G$ then
  $\hquin{A}{P}{\newSwitchClause{\coma}{\comb}{C}{D}}{Q}{G}$ is also
  correct.
\end{lemma}



Like for \cref{lem:aslongas:strong}, since $C$ is a guarantee of
$\coma$, $\coma$ is not interrupted, and thus
$\switchClause{\coma}{\comb}{C}$ behaves like $\coma$.
The condition $A\land Q \limply G$ can always be satisfied:
using rule $(\limprule)$, we can always assume that this 
condition is true.

\subsubsection{Combining Guarantees under a Weak Assumption}
The following definition is used to characterize the behavior
of $\asLongAsClause{\coma}{C}$ from that of $\coma$ when the
assumption does not imply $C$.

\begin{definition}[interruption-extension]
  \label{def:int-ext}
  We say that assertion $D$ is an \emph{interruption-extension} 
  (\emph{int-ext} for short) of
  assertion $C$ for program $\coma$ from assertion $P$ along assertion
  $A$ if, for all $\store \vDash P$, $\sigma$ valid for $\coma$ from
  $\store$, and $(i,t) \in \dom(\sigma)$, if for all $(i',t') \in
  \dom(\sigma)$ such that $(i',t') < (i,t)$, $\sigma(i',t') \vDash A
  \land C$, and $\sigma(i,t) \vDash A$, then $\sigma(i,t) \vDash D$.
\end{definition}

This definition resembles the \emph{safety} part of correctness in
\cref{def:hoare:quintuples} and states that, if $C$
holds during the execution of $\coma$, except maybe
at the end, then $D$ holds at the end.

\begin{lemma}[weak as-long-as rule]
  \label{lem:aslongas-weak}
  This rule is correct:
  \[
    \AxiomC{$\hquin{A \land C}{P}{\coma}{Q}{G}$}
    \RightLabel{$(\untilrule)$}
    \UnaryInfC{$\hquin{A}{P}{\asLongAsClause{\coma}{C}}{(Q \land C) \lor
      (D \land \neg C)}{D}$}
    \DisplayProof
  \]
  where $D$ is an int-ext of $G
  \land C$ for $\coma$ from $P$ along $A$.
\end{lemma}

The intuition is as follows.
If $C$ holds at all times, then $\coma$ is not interrupted, and the 
assumption of this rule applies, so $Q$ and $G$ can be guaranteed.
Otherwise $\coma$ is interrupted, in which case, the assumption
guarantees that $C$ and $G$ are true at all times except at the very
last time.
By definition of an int-ext, $D$ is then guaranteed at
all times.





\begin{lemma}[Hoare rule for $\newSwitchClause{\coma}{\comb}{C}{D}$ under
  weak assumption]
  \label{lem:switch-weak}
  If:
    $\hquin{A}{P}{\coma}{Q}{G}$ and
      $\hquin{A'}{P'}{\comb}{Q'}{G'}$ are correct,
    $E$ is an int-ext of $G \land C$ for $\coma$
      from $P$ along $A'$,
    $E \limply P' \land G'$, $Q \limply Q'$, and
      $A' \land C \limply A$, and $Q' \limply G'$,
  then $\hquin{A'}{P}{\newSwitchClause{\coma}{\comb}{C}{D}}{Q'}{G'}$ is
  correct.
\end{lemma}

Contrary to \cref{lem:switch-strong}, $\coma$ may be interrupted.
Intuitively, either $C$ holds at all times, in which case the
assumptions ensure $Q$ and $G$, which can be weakened to
$Q'$ (using $Q \limply Q'$) and $G'$ (using $G \land C \limply E$ and
$E \limply G'$).
Otherwise, $\coma$ is interrupted, and we can only ensure that $D$
holds.
It is then crucial that $E \limply P'$ to ensure that after
$\coma$ the system ends in a state from which $\comb$ can ensure $Q'$
and $G'$.
As for \cref{lem:switch-strong}, we can always assume that $Q' \limply
G'$ holds.

\improvement{rule with more than two levels}

\begin{example}[one-way traffic, proving]
  We want to prove that the Hoare quintuples
  in~\eqref{eq:example-hoare} are correct.
  First, we prove the Hoare quintuples 
    $\hquin{A}{P}{\coma}{Q}{G'}$
    and
    $\hquin{A'}{G'}{\comb}{Q'}{G'}$,
  where $G' \equiv (\xpov - x > \dRSS(\vpov, v))$ is a strengthening
  of $G$.
  The proof that such Hoare quintuples are correct is too long for
  the paper and strongly resembles the proof
  in~\cite[Appendix~A]{Hasuo22_GARSS}, which involves using
  the $\RSS$ distance as an explicit invariant.
  We then simply use \cref{lem:switch-strong} and
  \cref{lem:switch-weak} (with $E \equiv G'$) to prove the quintuples
  in~\eqref{eq:example-hoare}.

\end{example}

\subsubsection{Safety of Advanced Controllers}

\label{sec:ac-bc}

Here, we consider a general \emph{advanced controller} $\AC$, modeled
as a program $\coma$, whose behavior is unknown and thus does not
come with any guarantees.
We want to make this controller safe by coupling it with a 
\emph{baseline controller} $\BC$, modeled as a program $\comb$ for
which we assume some safety guarantee.
Our goal is to design switching conditions $C$ and $D$ such that
$\newSwitchClause{\coma}{\comb}{C}{D}$ satisfies a guarantee similar
to that of $\BC$.

Since $\coma$ is general, $C$ has to be designed
in such a way that: 1) $G$ holds during the whole execution of
$\asLongAsClause{\coma}{C}$, 2) when $C$ is violated, $P$ holds (so
that $\comb$ can be run with some guarantee).

\begin{lemma}[safety of an advanced controller]
  \label{lem:ac-bc}
  If $\hquin{A}{\top}{\coma}{\top}{\top}$ and
  $\hquin{A}{P}{\comb}{Q}{G}$ are correct and $G \land (P \lor (C \land D))$ is an
  int-ext of $G \land C$ for $\alpha$ from $P'$ along $A$ and $D \limply Q'$,
  then $\hquin{A}{P'}{\newFallbackClauseFull{\coma}{\comb}{C}{D}{0.7}}{\top}{G}$ is
  correct.
\end{lemma}


\section{Case Study: Pull Over Scenario}
\label{sec:pullover}



In~\cite{Hasuo22_GARSS}, we consider a complex pull over
scenario with several lanes and $\POV{}$s, depicted in
\cref{fig:pulloverIntro}.
%
This scenario is modeled using $y$ and $v$ to denote the position
and velocity of $\SV$, as well as $y_i$ and $v_i$ to denote those of
$\POV{i}$ (for $i = 1,2,3$).
We also use $l$, which is a half integer, to describe the current lane
of $\SV$: when it is in Lane~$i$, then $l = i$, and when it is
changing lane from Lane~$i$ to Lane~$i+1$, then $l = i+0.5$.
Similarly, the lane in which $\POV{i}$ runs is denoted $l_i$.
We consider that there is a collision between $\SV$ and $\POV{i}$ if
$\abs{l-l_i} \leq 0.5$ and $y - 2\ell \leq y_i \leq y$ (where $\ell$
is the length of a vehicle, and the point of reference of vehicles is
the front for $\SV$ and the rear for $\POV{}$s).
Finally, the minimal and maximal legal speeds are denoted $\vmin$ and
$\vmax$, while the position of the goal on Lane~3 is denoted $\ytgt$.
We also define a framework that
allows us to design a program $\coma_{\GARSS}$ that achieves stopping
at $\ytgt$ on the
shoulder lane while avoiding all collisions with $\POV{}$s, under the
assumption that all $\POV{}$s have constant speeds.
Compared to~\cite{Hasuo22_GARSS}, we add a flag $f$ to
$\coma_{\GARSS}$, which is set to $1$ when the vehicle is going to
turn left and is $0$ otherwise.

In~\cite{Hasuo22_GARSS}, we are interested in proving that we
can achieve the goal of stopping at $\ytgt$ on Lane~3, modeled as
\begin{math}
   \Goal \equiv (l = 3 \land y = \ytgt \land v = 0)
\end{math}.
We make some physical assumptions and constant speed of $\POV{}$s:
\begin{math}
  \Env
  \equiv \textstyle \big( \bigland_{i=1}^3 \vmin \leq v_i \leq \vmax
  \big)
\end{math}
and
\begin{math}
  \Env_a
 \equiv \textstyle \big( \bigland_{i=1}^3 a_i = 0 \big) 
\end{math}.
All the while, we want to show that we respect the RSS distance, which
is both the safety and precondition of $\coma_{\CARSS}$, and is
modeled as
\begin{math}
   \textstyle \Safe \equiv P'
  \equiv \big( \bigland_{i=1}^3 \aheadSL_i \limply y_i - y > \dRSS(v_i,v) \big)
\end{math}
where $\aheadSL_i = (\abs{l-l_i} \leq 0.5 \land y \leq y_i)$.

The framework allows us to
prove~\cite[Example~IV.12]{Hasuo22_GARSS}:
\begin{example}
  \label{ex:pullover:safe-constant}
  The following Hoare quintuple is correct:
  \[
    \hquin{\Env \land \Env_a}{P}{\coma_{\GARSS}}{\Goal}{\Safe}
    \rlap{,}
  \]
  where $P$ is some assertion computed using the framework.
\end{example}

\begin{auxproof}
  \begin{remark}
    This is not exactly what~\cite[Example IV.12]{Hasuo22_GARSS}
    proves, since it uses Hoare quadruples where $\Env$ and $\Safe$ are
    indistinguishable.
    However, the proof of the quintuple is structurally the same. 
  \end{remark}
\end{auxproof}

Here, we want to prove that, even if the $\POV{}$s change their speeds,
we can fallback on collision-avoiding RSS $\coma_{\CARSS}$ to avoid
collision (but losing goal achievement).
The assumption on acceleration $\Env_a$ is thus dropped, but we add
the assumption that $\POV{}$s do not crash into $\SV$ from behind
($\SV$ is not responsible for such collisions anyway), encoded as
\begin{math}
   \textstyle
  \Env' \equiv
  \big( \bigland_{i=1}^3 \behindSL_i \limply y - y_i > \dRSS(v,v_i) + 2\ell \big),
\end{math}
where $\behindSL_i \equiv (\abs{l-l_i} \leq 0.5 \land y_i
< y)$.
\begin{example}[safety of the pullover scenario]
  The following Hoare quintuples are correct:
  \label{ex:pullover-fallback}
  \begin{align}
    \hquin{\Env \land \Env_a}{P}{\newFallbackClauseFull{\coma_{\GARSS}}{\coma_{\CARSS}}{C}{}{1.2}}{\Goal}{\Safe}
    \label{eq:pullover:quintuple-strong} \\
    \hquin{\Env \land \Env'}{P}{\newFallbackClauseFull{\coma_{\GARSS}}{\coma_{\CARSS}}{C}{}{1.2}}{\top}{\Safe'} \rlap{,}
    \label{eq:pullover:quintuple-weak}
  \end{align}
  where $C$ is the switching condition and $\Safe'$ is the
  mild variant of $\Safe$ defined as:
  \begin{gather*}
    C
      \equiv \textstyle \big( P' \land (f = 1 \limply
      P'[l+0.5/l]) \land \bigland_{i=1}^3 a_i = 0 \big)\\
    \Safe'
      \equiv \textstyle \big( \bigland_{i=1}^3 \aheadSL_i \limply y_i -
      y \geq \dRSS(v_i,v) \big) \rlap{.}
  \end{gather*}
  Note that $\Safe'$ does prevent collisions since, if $y_i = y$, then
  $\dRSS(v_i,v) = 0$, which implies that $v < v_i$ or $v = v_i = 0$.
\end{example}

\begin{proof}
  We can prove \eqref{eq:pullover:quintuple-strong} directly using
  \cref{ex:pullover:safe-constant} and \cref{lem:switch-strong}.
  To prove \eqref{eq:pullover:quintuple-weak}, we need an
  int-ext $E$ of $\Safe \land C$ along $\Env \land \Env'$
  and use \cref{lem:switch-weak}.
  Taking $E \equiv \Safe'$
  gives an int-ext as desired.
  The proof that $E$ is as desired is semantic and relies heavily on the
  assumption $\Env'$ to keep other vehicles from creating a collision
  with $\SV$ from behind.
\end{proof}

\section{Case Study: Simplex Architecture}
\label{sec:simplex}

Until now, we studied the fallback
$\newSwitchClause{\coma}{\comb}{\assertc}{\assertd}$, which allows to interrupt
$\coma$ to start $\comb$ when $\assertc$ becomes
false.
This allows to model the interruption of $\AC$ by the decision module
to start $\BC$ when $\AC$ is deemed unsafe.
However, in the simplex architecture, there is also the possibility to
start $\AC$ again when the situation allows it. 
We encapsulate this as a new constructor
 $\newSimplexClause{\coma}{\comb}{\assertc}{\assertd}{\assertc'}$ whose valid traces
 are defined as follows.
\begin{definition}[trace semantics of simplex]
  \label{def:traces:simplex}
  $\store,\trace \vDash \newSimplexClause{\coma}{\comb}{C}{D}{C'}$ iff
  there exists $n \in \Nbar$ such that $\sigma = \concatfinsym_{i=0}^n
  \sigma_i$ with:
  \begin{itemize}
    \item for all $i < n$, $\sigma_i \nvDash C \land D$ if $i$ even,
      and $\sigma_i \nvDash C'$ if $i$ odd,
    \item $\sigma_n \vDash C \land D$ if $n$ even, $\sigma_n \vDash
      C'$ if $n$ odd (if $n < +\infty$),
    \item for all $i \leq n$ ($i < n$ if $n = +\infty$),
      $\traceend(\sigma_{i-1}),
      \sigma_i \vDash \asLongAsClause{\coma}{C}$ if $i$ is even,
      $\traceend(\sigma_{i-1}), \sigma_i \vDash
      \asLongAsClause{\comb}{C'}$ if $i$ is odd.
  \end{itemize}
\end{definition}

\begin{remark}\label{rem:simplexByAsLongAs}
As for the fallback, we could have defined the simplex
$\newSimplexClause{\coma}{\comb}{\assertc}{\assertd}{\assertc'}$
using the existing constructors. 
\end{remark}

\subsection{Partial Correctness}

Since the semantics of the simplex is complex, we divide the
derivation of Hoare quintuples in two steps.
We first focus on partial correctness of the simplex, which is
similar to \cref{def:hoare:quintuples}:
\begin{definition}
  A Hoare quintuple is \emph{partially correct} if it satisfies
  \cref{def:hoare:quintuples}, except possibly for
  \emph{termination}.
\end{definition}
%

The partial correctness of Hoare quintuples for the simplex architecture 
can be deduced from the (partial) correctness of its components:
\begin{lemma}
  \label{lem:partial}
  If all the hypotheses of \cref{lem:switch-weak} (partially)
  hold with $A' = A$, $Q' = Q$, and $G' = G$, and
  there exists $E'$ an int-ext of $C'$ for $\comb$
  from $P'$ along $A$ such that $E' \land \neg C' \limply P \land
  C$,
  then $\hquin{A}{P}{\newSimplexClause{\coma}{\comb}{\assertc}{\assertd}{\assertc'}}{Q}{G}$
  is partially correct.
\end{lemma}
Here, we say that a hypothesis of \cref{lem:switch-weak} holds
partially if it is a Hoare quintuple that is partially correct, or it
is a regular hypothesis and it simply holds.


\todoproof{
\begin{lemma}
  \label{lem:partial}
  If 
  \begin{itemize}
    \item $\hquin{A\land\closure{C}}{P \lor (\lnot\assertd\land\closure{\assertd})}{\coma}{C \limply Q}{G}$
      is partially correct,
    \item $\hquin{A\land\closure{D}}{\lnot\assertc\land\closure{\assertc}}{\comb}{D \limply Q}{G}$
      is partially correct,
    \item $A \land \lnot \assertd \land \closure{\assertd} \limply \assertc$,
      $A \land \lnot \assertc \land \closure{\assertc} \limply \assertd$, and
      $P \limply \assertc$,
  \end{itemize}
  then $\hquin{A}{P}{\simplexClause{\coma}{\comb}{\assertc}{\assertd}}{Q}{G}$
  is partially correct.
\end{lemma}

\begin{proof}
  Let $\sigma_1, \ldots, \sigma_n, \ldots$ be a trace valid for 
  $\simplexClause{\coma}{\comb}{\assertc}{\assertd}$ from a store $\store \vDash P$, 
  as above and assume that for all $n < m$, $\sigma_n \vDash A$ and
  $\restr{\sigma_m}{i,t} \vDash A$.
  By induction and using lemmas about as-long-as program, we can prove that:
  \begin{itemize}
    \item if $2n < m$, $\sigma_{2n} \vDash \asserta\land\closure{\assertd}\land G$
      and $\traceend{(\sigma_{2n})} \vDash \lnot\assertd\land\closure{\assertd}\land\assertc$,
    \item if $2n+1 < m$,  $\sigma_{2n+1} \vDash \asserta\land\closure{\assertc}\land G$, 
      and $\traceend{(\sigma_{2n+1})} \vDash \lnot\assertc\land\closure{\assertc}\land\assertd$,
    \item if $m = 2n$, $\restr{\sigma_m}{i,t} \vDash \asserta\land\closure{\assertd}\land G$, and
    \item if $m = 2n+1$, $\restr{\sigma_m}{i,t} \vDash \asserta\land\closure{\assertd}\land G$.
  \end{itemize}
  This implies the guarantee. Furthermore, if $\sigma$ is finite, let $m$ be the last index of 
  the sequence, $\sigma_m = \bigl((t_j,h_j) \bigr)_{j < i+1}$, $t = t_i$.
  By definition of traces for the simplex, we then know that:
  \begin{itemize}
    \item either $m = 2n$, and $\sigma_m \vDash \assertd$ and 
      $\traceend{(\sigma_{m-1})}, \sigma_m \vDash \comb$,
    \item or $m = 2n+1$, and $\sigma_m \vDash \assertc$ and 
      $\traceend{(\sigma_{m-1})}, \sigma_m \vDash \coma$.
  \end{itemize}
  In either case, we know that $\traceend{(\sigma_m)} \vDash Q$, which is the postcondition.
\end{proof}
}

\subsection{Termination}

The main challenge to prove the total correctness of the simplex is its termination.
Indeed, one possible way for the simplex to not terminate is when 
the system oscillates between $\coma$ and $\comb$ infinitely often.
In this section, we show two ways to design a simplex ensuring its 
termination.
Assume given programs $\coma$ and $\comb$ and formulas $\assertc$,
$\assertd$, and $\assertc'$ on which we build our simplex.

\subsubsection{Counters}

Assume given a fresh variable $c$ and a natural number $N$.
The idea is to use $c$ as a counter to count the number of switches.
To ensure termination, we require the switch to be done at most $N$ times as follows.
Formally, we transform the programs and assertions as follows:
$\widetilde{\coma} \equiv \coma$,
$\widetilde{\comb} \equiv (\assignClause{c}{c+1} ; \comb)$,
$\widetilde{\assertc} \equiv \assertc$,
$\widetilde{\assertd} \equiv \assertd$, and 
$\widetilde{\assertc'} \equiv (c \leq N \limply \assertc')$.

\begin{lemma}
  \label{lem:total:counter}
  If
  $\hquin{A}{P}{\newSwitchClause{\coma}{\comb}{\assertc}{\assertd}}{Q}{G}$ is 
  correct,
  then $\hquin{A}{P}{\newSimplexClause{\widetilde{\coma}}{\widetilde{\comb}}{\widetilde{\assertc}}{\widetilde{\assertd}}{\widetilde{\assertc'}}}{Q}{G}$
  is correct.
\end{lemma}

\subsubsection{Timers}


Given $t$ and $\coma$, we define $\coma_t$ as
the program obtained from $\coma$ by replacing all the programs
$\dwhileClause{A}{\odeClause{\vars}{\funs}}$ by 
$\dwhileClause{A}{\odeClause{\vars}{\funs}, \dot{t} = 1}$.
We then transform the programs and assertions into 
$\widetilde{\coma} \equiv \coma_t$,
$\widetilde{\comb} \equiv (\assignClause{t_0}{t} ; \comb_t)$,
$\widetilde{\assertc} \equiv \assertc$,
$\widetilde{\assertd} \equiv \assertd$, and 
$\widetilde{\assertc'} \equiv ((t - t_0 \geq \epsilon \land t \leq T)
\limply \assertc')$,
where $t_0$ is a free variable, such that the value of $t - t_0$
describes the time $\comb$ has been executing for since the last
switch.


\begin{lemma}
  \label{lem:total:timer}
  If all the hypotheses of \cref{lem:partial} hold, then
  $\hquin{A}{P}{\newSimplexClause{\widetilde{\coma}}{\widetilde{\comb}}{\widetilde{\assertc}}{\widetilde{\assertd}}{\widetilde{\assertc'}}}{Q}{G}$
  is correct.
\end{lemma}
\todoproof{
  By Lemma~\ref{lem:partial}, we know that 
  $\hquin{A}{P}{\simplexClause{\coma}{\comb}{\assertc}{\assertd}}{Q}{G}$
  is partially correct.
  Since $\assertc \limply\assertc'$ and $\assertd \limply\assertd'$ and $c$ and $c_0$ are 
  fresh variables, we can deduce that 
  $\hquin{A}{P\land c = 0}{\simplexClause{\coma'}{\comb'}{\assertc'}{\assertd'}}{Q}{G}$
  is also partially correct.
  It remains to prove the termination.
  We know that $\lnot\assertc'\limply c \leq N$ and $\lnot\assertd'\limply c \leq N$.
  Furthermore:
  \begin{itemize}
    \item for counters, after each switch, $c$ is increased by $1$ and $c$ is forced to be $0$ 
      initially by the precondition, so there cannot be more than $N$ switches,
    \item for timers, $\lnot\assertc'$ and $\lnot\assertd'$ both imply $c - c_0 \geq \varepsilon$, 
      which means that at every switch from $\comb$ to $\coma$, $c$ is increased of 
      at least $\varepsilon$. Consequently, there cannot be more than 
      $2N/\varepsilon$ switches.
  \end{itemize}
  Consequently, a valid trace $\sigma$ for $\simplexClause{\coma'}{\comb'}{\assertc'}{\assertd'}$
  from $\store \vDash P\land c = 0$ is of the form 
  $\sigma_1, \ldots, \sigma_m$. It remains to prove that $\sigma_m$ is finite, which is true
  by the validity assumptions on $\coma$ and $\comb$. 
}

\section{Case Study: Layered Simplexes}
\label{sec:caseStudyLayered}



In this section, we show how to combine more than two controllers into
a single simplex architecture to achieve different guarantees under
different circumstances.

Precisely, we use the example of an $\AC$ with two $\BC$s: a
first one running $\coma_{\GARSS}$ and a second one running
$\coma_{\CARSS}$.
We want them to start running when some slightly conservative
preconditions become violated (so that we can guarantee that they
achieve their goal).
We denote by $C$ conservative precondition for $\coma_{\CARSS}$ and by
$D$ that of $\coma_{\GARSS}$.
The architecture is described in \cref{fig:introLayered}, and modeled as
\[
  \comalayer \equiv
  \newSimplexClauseFull{\coma_{\AC}}{\coma_{\BC}}{D}{\Goal}{D'}{1.2},
  \quad
  \coma_{\BC} \equiv
  (\newSimplexClauseFull{\coma_{\GARSS}}{\coma_{\CARSS}}{C}{}{C'}{1.5})
  \rlap{.}
\]


Our goal for the rest of this section is to design $C'$ and $D'$ such
that $\comalayer$ satisfies some guarantees derived from those of
$\GARSS$ and $\CARSS$.
We fix two positive reals $\epsilon$ and $\epsilon'$ such that
$\epsilon < \epsilon'$, which we will use as margins for conservative
preconditions.
We generalize the $\CARSS$ precondition $P'$ with margins as follows:
\begin{math}
   \textstyle
  P'(\epsilon) \equiv \bigland_{i=1}^3 (\aheadSL_i \limply y_i - y
  > \dRSS(v_i,v) + \epsilon) 
\end{math}.

The switching conditions for
$\coma_{\BC}$ are:
\small
\[\begin{array}{c}
  C
  \textstyle \equiv \big( P'(\epsilon) \land (f=1 \limply P'(0)[l+0.5/l]) \land \bigland_{i=1}^3 a_i = 0 \big) \\
  C'
  \textstyle \equiv \neg \big( P'(\epsilon') \land (f=1 \limply P'(\epsilon)[l+0.5/l]) \land\bigland_{i=1}^3 a_i = 0 \land P \big)
\end{array}\]
\normalsize

\begin{example}[safety of $\coma_{\BC}$]
  By the same reasoning as in \cref{ex:pullover-fallback} (with
  margins), we get that
  \eqref{eq:pullover:quintuple-strong} and
  \eqref{eq:pullover:quintuple-weak} are correct (for this more
  conservative $C$).
 By \cref{lem:partial} (with $\top$ as the
  int-ext), we get that $\coma_{\BC}$
  satisfies the same quintuples.
\end{example}


Similarly, the $\GARSS$ precondition $P$ is a Boolean
combination of inequalities $f(\overrightarrow{x}) >
g(\overrightarrow{x})$.
We generalize it to $P(\epsilon)$, where inequalities have been
strengthened into $f(\overrightarrow{x}) > g(\overrightarrow{x}) +
\epsilon$.
Note that $P$ is derived in such a way that it respects the RSS
distance, so in particular it implies $P'$.
The switching condition for
$\simplexClause{\coma_{\AC}}{\coma_{\BC}}{D}{D'}$ are:
\begin{displaymath}
\begin{array}{l}
   D \equiv \big( P(\epsilon) \land P(0)[l+0.5/l] \land P(0)[l-0.5/l] \big) \\
  D'
  \equiv \neg \big( D \land P(\epsilon') \land P(\epsilon)[l+0.5/l] \land P(\epsilon)[l-0.5/l] \big)
    \rlap{.}
\end{array}
\end{displaymath}

\begin{example}[safety of $\comalayer$]
  Let us assume that, in $\coma_{\AC}$, assignments to $l$ are only of
  the form $\assignClause{l}{l+0.5}$ or $\assignClause{l}{l-0.5}$,
  which models the fact that $\SV$ cannot ``skip'' lanes.
  By \cref{lem:ac-bc} (for partial correctness), we get that
  \begin{gather*}
\begin{array}{l}
     \hquin{\Env \land \Env_a}{D}{\switchClause{\coma_{\AC}}{\coma_{\BC}}{D}}{\Goal}{\Safe}\\
    \hquin{\Env \land \Env'}{D}{\switchClause{\coma_{\AC}}{\coma_{\BC}}{D}}{\top}{\Safe'}
\end{array}  
\end{gather*}
  are partially correct.
  By \cref{lem:partial} (again with $\top$ as the
  int-ext), the Hoare quintuples 
    $\hquin{\Env \land \Env_a}{D}{\comalayer}{\Goal}{\Safe}$ and 
    $\hquin{\Env \land \Env'}{D}{\comalayer}{\top}{\Safe'}$
  are
  partially correct:
\end{example}


\section{Experiments}
\label{sec:exp}

We conducted experiments to evaluate the practical values of the
proposed framework.
The experiments used the setting of \cref{sec:caseStudyLayered}, where
1) the driving scenario is the pull over one (\cref{fig:pulloverIntro}),
2) \SV{} is equipped with the layered simplexes in which CA-RSS
safeguards GA-RSS, and 3) the $\POV{}$s may change speed.
We posed the following research questions.



\textbf{RQ1 (weak guarantee).} 
Do the layered simplexes successfully ensure safety, 
even if \POV{}s change speed? 


This is where the CA-RSS component should act to avoid collision. Since the GA-RSS assumption is violated, we should not expect that the GA-RSS goal (namely reaching the designated stopping position) is ensured. Safety is mathematically established in~\cref{sec:caseStudyLayered}, but we want to experimentally confirm.

\textbf{RQ2 (strong guarantee).} 
Do the layered simplexes successfully ensure goal achievement (reaching the stopping position on the shoulder), 
in case \POV{}s \emph{do not} change speed?


The GA-RSS rule for this scenario is designed to ensure this~\cite{Hasuo22_GARSS}, and its assumption is satisfied in this setting. Therefore we want to confirm---although it is mathematically established in~\cref{sec:caseStudyLayered}---that the additional CA-RSS simplex does not tamper the operation of the GA-RSS simplex. 

\textbf{RQ3 (best-effort goal achievement).} 
 Can \SV{} reach the  stopping position, even when \POV{}s change their speed?

 Our layered simplex architecture tries to give the control back from CA-RSS to GA-RSS, and then to \AC{}, when possible. This is  in order to minimize the interference of  more restrictive controllers. We would like to see that this design indeed results in best-effort goal achievement of GA-RSS.


As \AC{} of our controller, 
we used a prototype
planner (a research prototype  provided by Mazda Motor Corporation; it is  unrelated to any of its products)  based on the algorithm in~\cite{mcnaughton2011motion}. 
\AC{} is a sampling-based controller that, at each time step, generates a large
number of candidate short-term paths and chooses the best in terms of a predetermined cost
function. 

\begin{auxproof}
 Each \POV{} having a constant speed is a central component of our GA-RSS environmental assumption. In the experiments, we modeled this condition using a \emph{brake light}: there is a global  variable $\mathtt{brakeLight}_{i}$ for each $\POV{i}$ in the simulator; $\POV{i}$ brakes if $\mathtt{brakeLight}_{i}$ is $\true$ and it maintains its speed otherwise (for simplicity we assume \POV{}s do not accelerate); and \SV{} deems the constant speed assumption is no longer true once $\mathtt{brakeLight}_{i}$ is $\true$.
\end{auxproof}

We ran simulations under  settings that differ in 1) the stopping position $\ytgt$, 2) the initial positions and velocities of \SV{} and \POV{}s,
and 3) whether and when \POV{}s brake. 

\begin{wrapfigure}[15]{r}{0pt}
\includegraphics[width=.13\textwidth]{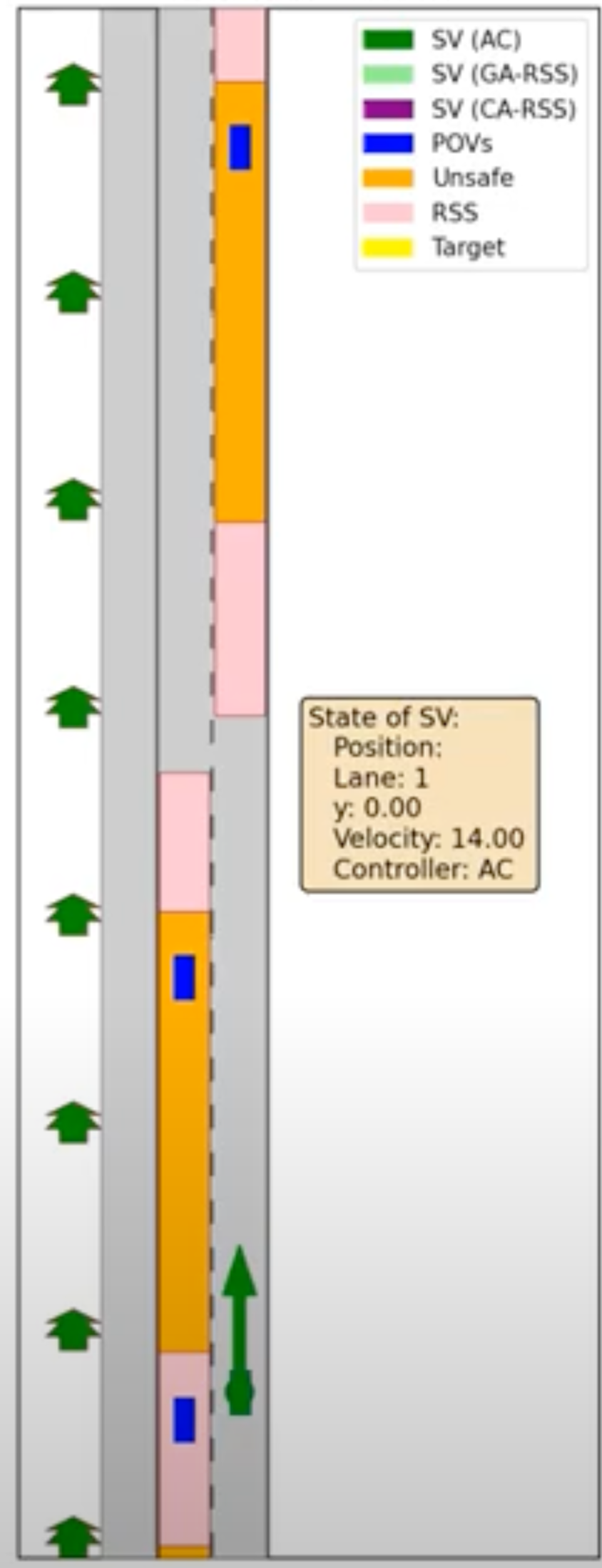}
\end{wrapfigure}
These simulations answered RQ1 and RQ2  positively: there were no collisions; and \SV{} reached the stopping position on the shoulder in all those settings where \POV{}s do not brake.

To address RQ3, we exhibit two notable instances, in which 1) GA-RSS \BC{} is interrupted  due to \POV{2} braking, 2) GA-RSS \BC{} regains the control after \POV{2} stops braking, and 3) in the end, GA-RSS \BC{} successfully makes \SV{} reach the stopping position. 
These instances answer RQ3 positively: our layered simplexes switch back to less restrictive controllers when possible; this allows ADVs to pursue best-effort goal achievement while ensuring safety. 

In the first notable instance (the video is at \url{https://bit.ly/3HrKg3o}),  
vehicles are initially positioned as shown on the right, with 
$
y_{\POV{1}}=-2,
y_{\POV{2}}=30,
y_{\POV{3}}=90,
y_{\SV{}}=0,
\ytgt = 180
$ [\SI{}{\metre}]; the initial velocity is  \SI{10}{\meter\per\second} for all \POV{}s and \SI{14}{\meter\per\second} for \SV{}. We made \POV{2} brake  from \SI{1}{\second} to \SI{1.5}{\second}, at the rate \SI{-3}{\meter\per\second^2}.

In the simulation, \AC{} was initially in control, but GA-RSS \BC{} soon took over, engaging the proper response that accelerates and merges in front of \POV{1}. However, \POV{2} started braking while \SV{} was accelerating; this violates the GA-RSS assumption and thus made \SV{} follow CA-RSS \BC{} and brake in  Lane 1. When \POV{2} was done braking at time \SI{3.5}{\second}, the control was given back to GA-RSS \BC{}, which found that the same ``accelerate and merge in front of \POV{1}'' proper response is safely executable. The controller engaged the proper response, successfully merging in front of \POV{1} and reaching the stopping position.

In the second notable instance (the video is at \url{https://bit.ly/3wMRbPQ}), 
we used the same initial positions and velocities of the vehicles, setting $\ytgt=120$ [\SI{}{\metre}] and making \POV{2} brake from  \SI{1.5}{\second} to \SI{3.5}{\second} (that is longer), at the same rate.

 The simulation proceeded initially much like the first notable instance, but longer braking by \POV{2} made the original ``accelerate and merge in front of \POV{1}'' proper response'' proper response no longer safety executable. Therefore,  the control is given back to GA-RSS \BC{} after  \POV{2}'s  braking, the controller engaged a different proper response, namely the one that brakes and merges behind \POV{1}. This way \SV{} successfully reached the stopping position.


\begin{auxproof}
 \subsection{Notable instances}

 We noted two notable instance.
 \begin{itemize}
  \item In the first notable instance, \SV{} starts in a position in which it
        can merge between \POV{1} and \POV{2}. However \POV{2} will brake during
        the time interval $[1.5, 3.5]$, disrupting the precondition for merging
        between \POV{1} and \POV{2}. The concrete scenario instance parameters
        are as follows:
        \[
          \begin{array}{l}
            y_{1}(0) = -2, \ y_{2}(0) = 30, \ y_{3}(0) = 90, \ \xtgt = 120,\\
            v(0) = 14, \ v_{i}(0) = 10.
          \end{array}
        \]
        Initially the \AC{} is in control, however it does not accelerate enough
        to reach the target while merging between \POV{1} and \POV{2}. The \BC{}
        takes over to engage in accelerating. During this proper response,
        \POV{2} starts braking, which leads to the \BC{} braking as a proper
        response because the constant speed assumption for the \POV{}'s has been
        violated. When \POV{2} has finished braking, the precondition for
        merging in front of \POV{1} no longer holds. \AC{} is in control briefly,
        before \BC{} takes over again, this time to brake so as to ensure it
        merges behind \POV{1} in time for coming to a stop at the target in lane
        3.
  \item In the second notable instance, \SV{} starts in a position in which it
        can merge between \POV{1} and \POV{2}. It is disrupted by \POV{2}
        braking during the time interval $[1, 1.5]$, however after this the
        precondition for merging between \POV{1} and \POV{2} still holds.
        The concrete scenario instance parameters
        are as follows:
        \[
          \begin{array}{l}
            y_{1}(0) = -2, \ y_{2}(0) = 30, \ y_{3}(0) = 90, \ \xtgt = 180,\\
            v(0) = 14, \ v_{i}(0) = 10.
          \end{array}
        \]
        The simulation proceeds initially much like in the first notable
        instance. The difference is that after \POV{2} finishes braking, the
        precondition for the rule for merging in front of \POV{1} still holds.
        The proper response for accelerating and merging between \POV{1} and
        \POV{2} is activated and the \SV{} successfully reaches the target
        location.
 \end{itemize}
\end{auxproof}

\section{Conclusions}

We have defined a logic to formally define and prove properties of
safety architectures for ADVs.
We have applied it to the simplex and layered simplex architectures in
several case studies, and experimentally confirmed its usefulness.





%



\bibliographystyle{IEEEtran}
\bibliography{myrefs_shorten}


\end{document}